\documentclass[11pt]{article}

\usepackage{mathpazo}
\usepackage{authblk}
\usepackage[title]{appendix}

\usepackage[colorlinks, linkcolor=blue, citecolor=blue]{hyperref}           
\usepackage{color}
\usepackage{fullpage}
\usepackage[small,bf]{caption}
\usepackage[top=1in,bottom=1in,left=1in,right=1in]{geometry}
\usepackage{fancybox}

\usepackage{graphicx}
\usepackage{subfigure,verbatim}
\usepackage{booktabs} 

\usepackage{amsfonts,dsfont,amssymb,amsthm,amsmath,subfigure}       
\usepackage{algorithm}
\usepackage{verbatim}
\usepackage[noend]{algorithmic}

\usepackage{mathtools}

\newcommand\sbullet[1][.5]{\mathbin{\vcenter{\hbox{\scalebox{#1}{$\bullet$}}}}}

\newcommand{\p}[1]{\left(#1\right)}
\newcommand{\pp}[1]{\left[#1\right]}
\newcommand{\ppp}[1]{\left\{#1\right\}}

\newcommand{\ceil}[1]{\left \lceil #1 \right \rceil}

\newcommand{\calA}{{\cal A}}
\newcommand{\calB}{{\cal B}}
\newcommand{\calC}{{\cal C}}
\newcommand{\calD}{{\cal D}}
\newcommand{\calE}{{\cal E}}

\newcommand{\calG}{{\cal G}}

\newcommand{\sN}{{\mathsf{N}}}
\newcommand{\sM}{{\mathsf{M}}}
\newcommand{\sK}{{\mathsf{K}}}
\newcommand{\sT}{{\mathsf{T}}}
\newcommand{\bp}{{\mathbf{p}}}

\newcommand{\calP}{{\cal P}}

\newcommand{\calS}{{\cal S}}
\newcommand{\calT}{{\cal T}}

\newcommand{\calV}{{\cal V}}

\newcommand{\s}[1]{\mathsf{#1}}
\newcommand{\bE}{\mathbb{E}}
\usepackage{wrapfig}

\newcommand {\pr} {\mathbb{P}}

\floatname{algorithm}{Procedure}

\newtheorem{example}{Example}

\newtheorem{theorem}{Theorem}

\newtheorem{lemma}{Lemma}

\newtheorem{rmk}{Remark}
\newcommand{\ca}[1]{\mathcal{#1}}
\newcommand{\f}[1]{\mathbf{#1}}
\newcommand{\bb}[1]{\mathbb{#1}}

\floatname{algorithm}{Procedure}

\usepackage{multicol}

\begin{document}

\title{Learning User Preferences in Non-Stationary Environments}
\date{}
\author{Wasim~Huleihel\thanks{W. Huleihel is with the Department of Electrical Engineering-Systems at Tel-Aviv university, {T}el-{A}viv 6997801, Israel (e-mail:  \texttt{wasimh@tauex.tau.ac.il}).} ~~~Soumyabrata~Pal\thanks{S. Pal is with the Computer Science Department at the University of Massachusetts Amherst, Amherst, MA 01003, USA (email: \texttt{spal@cs.umass.edu}).}
~~~Ofer~Shayevitz\thanks{O. Shayevitz is with the Department of Electrical Engineering-Systems at Tel-Aviv university, {T}el-{A}viv 6997801, Israel (email: \texttt{ofersha@eng.tau.ac.il}).}}
\maketitle

\begin{abstract}
Recommendation systems often use online collaborative filtering (CF) algorithms to identify items a given user likes over time, based on ratings that this user and a large number of other users have provided in the past. This problem has been studied extensively when users' preferences do not change over time (static case); an assumption that is often violated in practical settings. In this paper, we introduce a novel model for online non-stationary recommendation systems which allows for temporal uncertainties in the users' preferences. For this model, we propose a user-based CF algorithm, and provide a theoretical analysis of its achievable reward. Compared to related non-stationary multi-armed bandit literature, the main fundamental difficulty in our model lies in the fact that variations in the preferences of a certain user may affect the recommendations for other users severely. We also test our algorithm over real-world datasets, showing its effectiveness in real-world applications. One of the main surprising observations in our experiments is the fact our algorithm outperforms other static algorithms even when preferences do not change over time. This hints toward the general conclusion that in practice, dynamic algorithms, such as the one we propose, might be beneficial even in stationary environments.
\end{abstract}

\section{Introduction}

Recommendation systems provide users with appropriate items based on their revealed preferences such as ratings. Due to their wide applicability, recommendation systems have received significant attention in machine learning and data mining societies \cite{Si2003,Rennie2005,Salakhutdinov2007,Bell:2007,Koren:2008,Salakhutdinov2008,Jahrer:2010}. In practice, recommendation systems often employ \emph{collaborative filtering} (CF) \cite{Ekstrand:2011}, for recommending (potentially) liked items to a given user by considering items that other ``similar" users liked. There are two main categories of CF algorithms: \emph{user-based}, e.g., \cite{Resnick94grouplens,Herlocker:1999,Bresler:2014,Bellogin:2012,Heckel:2017}, and \emph{item-based}, e.g., \cite{Gregory98,Sarwar:2001,Linden03,Bresler:2016}, and many references therein. User-based algorithms exploit similarity in the user space by recommending a particular user $u$ those items which were liked by other similar users. Item-based algorithms, in contrast, exploit similarity in the item space by recommending items similar to those which were liked in the past by a particular user. 

Prevalent recommendation systems typically operate in an online fashion where items are recommended to users over time, and the obtained ratings are used for future recommendations. Typically, the goal in such a problem is to maximize the number of likes revealed at any time. This problem has been studied extensively, e.g., \cite{Bresler:2014,Bresler:2016,Heckel:2017,Mina2019}, always under the assumption that user's preferences do not vary over time. In practice, however, temporal changes in the structure of the user's preferences are an intrinsic characteristic of the problem, since users change their taste occasionally \cite{Hariri2000,Jiahui2010,Joshua13,Hariri14,Karahodza14,Shi15,Basile2015ModelingSP,Liu:2015,Mukherjee2017ItemRW,Eskandanian2018DetectingCI}. Ignoring these changes results in recommendation algorithms which are doomed to failure in practice. This sets the goal of this paper: \emph{we aim to model and understand the effect of non-stationary environment on online recommendation systems}.

To that end, we introduce a novel latent probabilistic model for online non-stationary recommendation systems, and analyze the reward of an online user-based algorithm. Our model and certain elements of the algorithm are inspired by related static models and algorithms studied in \cite{Bresler:2014,Bresler:2016,Heckel:2017}. In a nutshell, each user in our model has a latent probability preference vector which describes the extent to which he/she likes or dislikes each item. Similar users have similar preference vectors (this will be defined rigorously in the following section). At a given time step $t=1,2,\ldots$, the algorithm recommends a single item to each user, typically different for each user, and with probability specified by the corresponding preference vector, the user likes or dislikes the recommended item. Following \cite{Bresler:2014,Bresler:2016,Heckel:2017} we impose the constraint that an item that has been rated by a user, cannot be recommended to the same user again. This is due to the fact that in many applications, such as, recommendation of movies or books, a rating often corresponds to consuming an item, and there is little point in, e.g., recommending a product that has been previously purchased in the past for a second time, at least not immediately. While in certain applications, this constraint might be unnecessary/irrelevant, it \emph{forces} our algorithm to exploit the user-item structure for collaboration. In any case, repeating the same recommendations only makes the problem easier algorithmically, and the results of this paper can be generalized to account for this case as well. Finally, to model the non-stationarity in the users' preferences, we allow users to change their user-type over time.

\noindent\textbf{Main Contributions.} Despite the success of CF, there has been no theoretical development to justify its effectiveness in non-stationary environment. The main contributions of this paper are two-fold. First, we introduce a novel model for general online non-stationary recommendation systems where we generalize the stationary model introduced in \cite{Bresler:2014} by allowing arbitrary shifts in user preferences over time. 
Our second main contribution is a theoretical analysis of a user-based CF algorithm that maximize the number of recommendations that users like. As time evolves, our CF algorithm randomly explores batches of items, one batch at a time, in order to learn users' preferences of new items in each batch. By splitting the space into optimal number of batches, our algorithms can start exploiting without having learned the preferences of the users regarding at all items. Furthermore, in each batch, our algorithm tests for variations, and once a change in the preference of a certain user is detected, that user is removed from the exploitation steps. Our results allow us to quantify the ``price of non-stationarity", which mathematically captures the added complexity embedded in changing users' preferences versus stationary ones. The proposed algorithm achieve near-optimal performance relative to an oracle that recommends all likable items first. Our findings, such as the scaling of the cold-start time on the various parameters, and the effect of non-stationary environment on recommendation, can inform
the design of recommendation algorithms, and refine our understanding of the associated benefits and costs while designing a practical recommendation system.

\noindent\textbf{Related Work.} While to the best of our knowledge, this is the first work that \emph{analytically} studies temporal changes in the users' preferences, theoretical results have been established for the stationary setting where there are no changes in the user preferences over time. We next briefly survey these prior works. One of the first initial asymptotic theoretical results concerning user-based CF were established in \cite{biau2010}. Most related to our approach are the setups and algorithms studied in \cite{Bresler:2014,Bresler:2016,Heckel:2017,Mina2019}. In particular, \cite{Bresler:2014} analyzed a user-based algorithm for online two-class CF problem in a similar setting to ours, while a corresponding item-based algorithm was analyzed in \cite{Bresler:2016}. A probabilistic model in an online setup was studied in \cite{Dabeer13}, and \cite{Dabeer12} study a probabilistic model in an offline setup, and derived asymptotic optimal performance guarantees for two-class CF problem. Theoretical guarantees for a one-class models were derived in \cite{Heckel:2017}. Another related work is \cite{Montanari12}, who considered online recommendation systems in the context of linear multi-armed bandit (MAB) problem \cite{Bubeck12}. 

Our setup relates to some variants of the MAB problem. An inherent conceptual difference between our setup and standard MAB formulation \cite{Thompson33} is that in our case an item can be recommended to a user just once, while in MAB an item (or arm) is allowed to be recommended ceaselessly. In fact, the solution principle for MAB is to explore for the best item, and once found, keep exploiting (i.e., recommending) it \cite{Auer02,Bubeck12}. This observation applies also to clustered bandits \cite{mannor12}, or, bandits with dependent arms \cite{Pandey07}. Specifically, in these formulations the arms are grouped into clusters, and within each cluster arms are correlated. It is assumed, however, that the assignment of arms to clusters is known, while in our setup this information is not available. Another related formulation of MAB is ``sleeping bandits" \cite{Kleinberg08}, where the availability of arms vary over time in an adversarial manner, while in our setup, the sequence of items that are available is not adversarial. Finally, a more recent related version is the problem of MAB with non-stationary rewards, e.g., \cite{Hartland07,Garivier08,Mannor09,Besbes,Besbes14,Karnin16,Luo2017EfficientCB,Cao19,Chen2019ANA,Auer19a,newreview1,newreview2,newreview3,newreview4,newreview5,newreview6}. This formulation allows for a broad range of temporal uncertainties in the rewards. While the motivation of this setup is similar to ours, due to the same reasons as above, the results and methods in these papers are quite different from ours. In particular, the main fundamental difficulty in our model compared to MAB literature lies in the fact that variations in the preferences of a certain user may affect the recommendations for other users severely.
\section{Model and Learning Problem}\label{sec:probsetting}
In this section we introduce the model and the learning problem considered in this paper. We start with the static setting where users' preferences do not change over time, and then generalize to the dynamic setting.

\noindent\textbf{Static Model.}
Consider a system with $\sN$ users and $\sM$ items. A user may ``like" ($+1$) or ``dislike" ($-1$) an item. At each time step, each user is recommended an item that he has not consumed yet. 
Each user $u\in[\sN]\triangleq\{1,2,\ldots,\sN\}$ is associated with a \emph{latent} (unknown) preference vector $\mathbf{p}_u\in\pp{0,1}^{\sM}$, whose entries $p_{ui}$ are the probabilities of user $u$ liking item $i\in[\sM]$, independently across items. We assume that an item $i$ is either \emph{likable} for user $u$, i.e., $p_{ui}>1/2$, or \emph{unlikable}, i.e., $p_{ui}<1/2$. The reward earned by the recommendation system up to any time step is the total number of liked items that have been recommended so far across all users (a precise definition will be given in the sequel). Accordingly, to maximize this reward, clearly likable items for the user should be recommended before unlikable ones. For a particular item $i$, recommended to a user $u$, the observed rating is a random variable $\mathbf{R}_{ui}$, such that $\mathbf{R}_{ui}=1$ with probability $p_{ui}$, and $\mathbf{R}_{ui}=-1$ with probability $1-p_{ui}$. The ratings are assumed random in order to model the fact that users are not fully consistent in their ratings. A CF algorithm operates as follows: at each time step $t=0,1,\ldots$, the algorithm recommends a single item $i=\pi_{u,t}\in[\sM]$ to each user $u\in[\sN]$, and obtains a realization of the binary random variable $\mathbf{R}_{ui}$ in response. Thus, if $\mathbf{R}_{ui}=1$, user $u$ likes item $i$, and $\mathbf{R}_{ui}=-1$ means that user $u$ does not like item $i$. Either way, as we explain and motivate in the Introduction, once rated by user $u$, item $i$ \emph{will not} be recommended to that user again.

To learn the preference of some user for an item, we need this user to rate that item. However, since we cannot recommend that item to the same user again, the only way to estimate the preferences of a user is through \emph{collaboration} (e.g., by making inferences from ratings made by other users). In order to make collaborative recommendations based on users' preferences we must assume some structure/relation between users and/or items. To that end, we study a latent model, in which users are clustered. Specifically, following \cite{Dabeer12,Dabeer13,Bresler:2014,Bresler:2016,Heckel:2017,Mina2019} (and many references therein), we assume that each user belongs to one of $\sK<\sN$ \emph{user-types}, where users of the same type have ``similar" item preference vectors. The number of types $\sK$ represents the heterogeneity in the population. This assumption is common and implicitly invoked by contemporary user-based CF algorithms \cite{Sarwar:2000}, which perform well in practice. Several empirical justifications for the clustering behavior in the user-item space can be found in, e.g., \cite{Sutskever09,Bresler:2014,Heckel:2017}.

There are many ways to define similarity between users. 
For example, in \cite{Das:2007,Dabeer12,Dabeer13,Mina2019} two users are considered of the same type if they have the same exact ratings, and these rating vectors are generated at random. While this model is perhaps unrealistic and does not capture challenges in real-world datasets, it allows for a neat theoretical analysis. A slightly more general and flexible model was proposed in \cite{Bresler:2014}. Here, two users $u$ and $v$ belong to the same type if their corresponding preference vectors $\bp_u$ and $\bp_v$ are the same, nonetheless, their actual ratings might be different. In \cite{Heckel:2017} this assumption was significantly relaxed by assuming instead that the preference vectors belonging to the same type are more similar (in terms of the magnitude of their inner product) than those belonging to other types. Roughly speaking, in this paper we follow this latter assumption, but we relax it even further. The precise statement of our assumptions will be given in the following section. This concludes the presentation of the static model. We next incorporate the non-stationary aspect.

\noindent\textbf{User Variations.} As explained in the introduction, our paper initializes the theoretical investigation of the situation where the preferences of users do not remain static but vary over time. To model this, we allow users to change their user-type over time. To wit, denote the type of user $u\in[\sN]$ at time $t\in[\sT]$ by $\calT_{u}(t)\in[\sK]$. In the sequel, $\calT_u(1)$, for any $u\in[\sN]$, designates the initial clustering of users into types. We assume that each user can change his type an \emph{arbitrary} number of times, but bound the total number of such variations. Specifically, we define two notions for variations:
\begin{align}
\mathsf{V_{1}}&\triangleq  \max_{u \in [\sN]}\sum_{t\in[\sT-1]}\mathds{1}[\calT_u(t) \neq \calT_u(t+1)],\label{eq:varbudg}
\end{align}
and
\begin{align}
\mathsf{V_{2}}&\triangleq\sN^{-1}\sum_{t\in[\sT-1]}\sum_{u \in [\sN]}\mathds{1}[\calT_u(t) \neq \calT_u(t+1)],\label{eq:varbudg2}
\end{align}
where $\mathds{1}[\cdot]$ is the indicator function. These definitions capture the constraint imposed on the non-stationary environment faced by the CF algorithm. In words, $\s{V}_1$ is the maximum number of variations over $\sT$ steps, while $\s{N}\s{V}_2$ is the total number of variations. In general, $\s{V_1}$ and $\s{V_2}$ are designed to depend on the time horizon $\sT$. It turns out that in order to obtain the tightest bound on the reward, both definitions are needed. It is clear that $\s{V_1}\leq\s{N}\cdot\s{V}_2$. 

In this paper, we consider the already non-trivial case where the values of (or, at least, upper bounds on) $\s{V}_1$ and $\s{V}_2$ are \emph{given} to the learner/algorithm in advance. This is inline with the various settings of classical results in the non-stationary MAB literature, e.g., \cite{Besbes,Besbes14,Luo2017EfficientCB}. Nonetheless, in a similar fashion to recent advances in the MAB literature \cite{Karnin16,Auer19a,Chen2019ANA}, it is an important, challenging, and of practical importance to propose and analyze algorithms which are oblivious to the number of variations (see Appendix~\ref{sec:concout}).

\noindent\textbf{Learning Problem and Reward.} Generally speaking, a reasonable objective for a CF algorithm is to maximize the expected reward up to time $\sT$, i.e.,
\begin{align*}
\overline{\mathsf{reward}}(\sT)\triangleq \bE\sum_{t\in[\sT]}\frac{1}{\sN}\sum_{u\in[\sN]}\mathbf{R}_{u\pi_{u,t}},
\end{align*}
where we note that the recommended item $\pi_{u,t}$ is a random variable because it is chosen by the CF algorithm as a function of previous responses to recommendations made at previous times. In this paper, however, we focus on recommending \emph{likable} items. Following \cite{Bresler:2014,Bresler:2016,Heckel:2017,Mina2019}, we consider the accumulated reward defined as the expected total number of liked items that are recommended by an algorithm up to time $\sT$, i.e.,
\begin{align}
    \mathsf{reward}(\sT)\triangleq \bE\sum_{t\in[\sT]}\frac{1}{\sN}\sum_{u\in[\sN]}\mathds{1}\pp{p_{u\pi_{u,t}}>1/2}.\label{eqn:rewardDef}
\end{align}
Note that the main difference between these two objectives is that the former prioritize items according to their probability of being liked, while the latter asks to recommend likable items for a user in an arbitrary order. This is sufficient for many real recommendation systems such as for movies and music (see, \cite[Sec. 2]{Bresler:2014} for a detailed discussion). We would like to emphasize that depending on the intended application, other metrics can be considered. For example, one may be interested in the number of actual ``clicks" rather then the number of ``likable" recommendations. Indeed, in some applications, the former is the measurable quantity. Nonetheless, we believe that our algorithms and techniques can handle such a criterion as well.



\section{Algorithms and Theoretical Guarantees}\label{sec:main}
In this section, we present our algorithm \textsc{Collaborative} which recommends items to users over time, followed by a formal theoretical statement for its performance. We start with a non-formal description of our algorithm. The pseudocodes are given
in Algorithms~\ref{algo:em}--\ref{algo:test}. The algorithm takes as input the parameters $(\alpha,\lambda,\s{T_{static}},\Delta_{\sT})$, which we shall discuss below.

We now describe the proposed algorithm. Below, we use ``$t$" to denote the time index at any step of the algorithm. Also, in the sequel, we use $\calP$ to denote an estimated partitioning of the users into clusters, i.e., $\calP$ is a collection of clusters, where each cluster refers to a set of users who have same estimated user type.
\vspace{-0.25cm}
\noindent\paragraph{Batches.} In Algorithm~\ref{algo:em}, the time horizon $\sT$ is partitioned into $\ceil{\sT/\Delta_{\mathsf{T}}}$ batches of size $\Delta_{\mathsf{T}}$ each, and denoted by $\{\calB_{b}\}_{b\geq1}$. We use $\tau$ to denote the time index at which each batch starts, namely, for the $b^{\s{th}}$ batch $\tau=(b-1)\Delta_{\mathsf{T}}$, for $b\in[ 1, \ceil{1/\Delta_{\mathsf{T}}}]$. At the beginning of each batch, we \emph{restart} Algorithm~\textsc{Recommend}, and run it for the entire batch. Also, at the beginning of each batch, we estimate an initial partition $\calP_0$ for the set of all users $[\s{N}]$ using Algorithm~\ref{algo:test}, which we describe bellow. 
At each subsequent time index $t$, the Algorithm~\textsc{Recommend} either runs a \emph{test for variations} with probability (w.p.) $\s{p}_{\sT}\propto\sqrt{\s{V}_1/\sT}$, or, explores-exploits w.p. $1-\s{p}_{\sT}$. 
\vspace{-0.25cm}
\noindent\paragraph{User partitioning in the batch.} The goal of Algorithm~\ref{algo:test} is to create a partition of the users into types. To that end, routine $\textsc{Test}(\s{T_{static}},\lambda,\calS_{t-1},\calP_0)$ recommends $\s{T_{static}}\in\mathbb{N}$ random items $\calT_{\s{test}}$ ($\left|\calT_{\s{test}}\right|=\s{T_{static}}$) to all users in $\calS_{t-1}\subseteq[\sN]$, initialized in each batch to be the set of all users $[\sN]$. Then, using the obtained responses $\ppp{\mathbf{R}_{ui}}_{u\in[\calS],i\in\calT_{\s{test}}}$, in the second and third steps of this algorithm a partition of the users is created. Specifically, for any pair of distinct users $u,v\in\calS_{t-1}$, we let $\s{X}_{u,v}$ be the number of items for which $u$ and $v$ had the same responses. Let $\s{E}_{u,v} = \mathds{1}\pp{\s{X}_{u,v}\geq\lambda\cdot\s{T_{static}}}$, for some $\lambda>0$. Accordingly, users $u$ and $v$ are inferred to have the same type if $\s{E}_{u,v}=1$. Subsequently, if there exists a valid partitioning $\calP$ over the set of users in $\calS_{t-1}$, which is consistent with the variables $\s{E}_{u,v}$, then we declare that $\calP$ is our estimated user-partition, otherwise, we place all users in the same group. This is true precisely when the graph with edge set $\s{E}_{u,v}$ is a disjoint union of cliques. Note that this partitioning procedure is equivalent to the cosine similarity test, declaring that two users $u$ and $v$ as being neighbors if their cosine similarity is at least as large as some threshold. The values of $\s{T_{static}}$ and $\lambda\in(0,1)$ are specified in Theorem~\ref{thm:1}. 
\vspace{-0.3cm}
\noindent\paragraph{Variations detection in the batch.} Given the partitions $\calP_0$ and $\calP$, in step 14 of Algorithm~\ref{algo:em} we compare those partitions in order to detect any variations using routine $\textsc{Variation}$. We show that if the user variations is not ``too large" in the corresponding batch, then it is possible to draw a one-to-one correspondence between the clusters in $\calP$ and $\calP_0$, and therefore, it is possible to identify the users who have changed their clusters, i.e., they are in a different cluster in $\calP$ than the one in $\calP_0$. All such users are declared as ``bad" users and are included in the set $\calV_t$. We update $\calS_t\leftarrow\calS_{t-1}\setminus\calV_t$. The users in $\calV_t$ will be excluded from future exploitation rounds. 

\begin{algorithm}[tb]
\caption{\textsc{Collaborative} Algorithm for recommending items.}\label{algo:em} 
\begin{algorithmic}[1]
\REQUIRE Parameters $(\alpha,\lambda,\s{T_{static}},\Delta_{\mathsf{T}})$.
\STATE Set index batch $b=1$.
\WHILE{$b\leq \ceil{\sT/\Delta_{\mathsf{T}}}$}
\STATE Set $\tau\leftarrow(b-1)\Delta_{\mathsf{T}}$.
\STATE Call \textsc{Recommend}$(\tau,\alpha,\lambda,\s{T_{static}},\Delta_{\mathsf{T}})$.
\STATE Set $b\leftarrow b+1$ and return to the beginning of Step 2.
\ENDWHILE
\end{algorithmic}
\end{algorithm}

\begin{algorithm}[tb]
\caption{\textsc{Recommend}$(\tau,\alpha,\lambda,\s{T_{static}},\Delta_{\mathsf{T}})$\label{algo:user_user_CF_noise}}
\begin{algorithmic}[1]
\STATE Select a random ordering $\sigma$ of the items $[\sM]$. 
\STATE Define $\s{p_R}=\sN^{-\alpha}$ and $\s{p_T}=\sqrt{\s{V_1}/(\sT\cdot\s{T_{static})}}$.
\STATE Let $t$ to be the time index at any step of the algorithm.
\STATE $\calP_0\leftarrow$\textsc{Test}($\s{T_{static}}$,$\lambda$,$[\sN]$).
\STATE Initialize $\calS_{\tau+\s{T_{static}}}\leftarrow[\sN]$.
\WHILE{$\tau+\s{T_{static}}<t\leq\min(\sT,\tau+\Delta_{\mathsf{T}})$}
\STATE Draw $\Theta\sim\mathsf{Bern}(\s{p_T})$.
\IF{$\Theta=0$}
\STATE $\calS_{t}\leftarrow\calS_{t-1}$.
\STATE$\sbullet[.75]$\textbf{ With probability} $\s{p_R}$: $\pi_{u,t}\leftarrow$ random item for all $u\in[\sN]$ that has not rated.
\STATE$\sbullet[.75]$\textbf{ With probability} $1-\s{p_R}$: $\pi_{u,t}\leftarrow$ item $i$ that user $u\in\calS_t$ has not rated and that maximizes score $\hat{p}_{ui}^{(t)}$ in \eqref{eqn:score}.
\ELSE
\STATE $\calP\leftarrow$\textsc{Test}($\s{T_{static}}$,$\lambda$,$\calS_{t-1}$).
\STATE $\calV_t \leftarrow$\textsc{Variation}$(\calP,\calP_0)$.
\STATE $\calS_t\leftarrow\calS_{t-1}\setminus\calV_t$.
\ENDIF 
\ENDWHILE
\end{algorithmic}
\end{algorithm}

\begin{algorithm}[tb]
\caption{\textsc{Test}($\s{T_{static}}$,$\lambda$,$\calS_{t-1}$) Algorithm for partitioning users.\label{algo:test}}
\begin{algorithmic}[1]
\STATE Recommend $\s{T_{static}}$ random items $\calT_{\s{test}}$ to all users in $\calS_{t-1}$.
\STATE For any $u\neq v\in\calS_{t-1}$, let $\s{X}_{u,v}$ be the number of items in $\calT_{\s{test}}$ for which $u$ and $v$ had the same responses, and let $\s{E}_{u,v} = \mathds{1}\pp{\s{X}_{u,v}\geq\lambda\cdot\s{T_{static}}}$.
\STATE Let $\calP$ be the valid partitioning over users consistent with the variables $\s{E}_{u,v}$. If such a partitioning does not exist, let $\calP \equiv \calS_{t-1}$.
\STATE Return $\calP$.
\end{algorithmic}
\end{algorithm}

\begin{algorithm}[tb]
\caption{\textsc{Variation}$(\calP,\calP_0)$ Algorithm for testing variations.\label{algo:var}}
\begin{algorithmic}[1]
\STATE For each cluster $\calC$ in $\calP$, find a cluster $\calC'$ in $\calP_0$ that shares at least half the users in $\calC$ i.e., $\left|\calC \cap \calC' \right| \ge \frac{|\calC|}{2}$. 
\STATE Establish a one-to-one mapping from clusters in $\calP$ to clusters in $\calP_0$ in this manner. If such a one-to-one mapping is not possible, return $\emptyset$.
\STATE Identify the set of users $\calV$ who belong to different clusters in $\calP$ and $\calP_0$.
\STATE Return $\calV$.
\end{algorithmic}
\end{algorithm}

\vspace{-0.3cm}
\noindent\paragraph{Exploration-Exploitation.} Since we restart the main algorithm in each batch, we focus on a particular batch $b$ in the explanation below. For ease of notation, we omit the batch index from all definitions. As mentioned above, w.p. $1-\s{p_T}$ our algorithm performs an exploration-exploitation routine. In such a case, with probability $\s{p_R}=\sN^{-\alpha}$ the algorithm randomly explores the space of items, and with complementary probability, $1-\s{p_R}$, the algorithm exploits by recommending those items that maximize a certain score. With some abuse of notation let $\mathbf{R}_{ui}^{(t)}\in\ppp{-1,0,1}$ be the observed rating of user $u$ to item $i$ up to time $t$ in the $b^{\s{th}}$ batch, where ``$0$" means that no rating has been given yet (in the $b^{\s{th}}$ batch). When exploiting, the algorithm evaluates empirical probabilities $\hat{p}_{ui}^{(t)}$, at time $t$, for user $u\in\calS_t$, and item $i$. These empirical probabilities are defined as follows,
\begin{align}
\hat{p}_{ui}^{(t)}\triangleq
\begin{cases}
\frac{\sum_{v\in\mathsf{neigh}(u,t)}\mathds{1}\pp{\mathbf{R}_{vi}^{(t)}=1}}{\mathsf{C}_{ut}},\ &\text{if }\mathsf{C}_{ut}>0\\
1/2,\ &\text{otherwise},
\end{cases}
\label{eqn:score}
\end{align}
where $\mathsf{C}_{ut}\triangleq\sum_{v\in\mathsf{neigh}(u,t)}\mathds{1}[\mathbf{R}_{vi}^{(t)}\neq0]$, and the ``neighborhood" of user $u\in\calS_t$ at time $t$ in the $b^{\s{th}}$ batch is $\mathsf{neigh}(u,t)\triangleq\calP_{0}(u)\cap\calS_t$, 
where $\calP_0(u)$ is the set of users in the same cluster as user $u$ in the initial partition $\calP_0$ created in the beginning of the batch. Note that the exploitation step at any time index $t$ is done only for those users which are present in $\calS_t$. Finally, we emphasize here that the empirical probabilities $\hat{p}_{ui}^{(t)}$ as well as the neighborhoods $\mathsf{neigh}(u,t)$, are all refreshed at the beginning of each batch; ratings from previous batches are ignored in the evaluation of these quantities.
\begin{rmk}
In practice, we can continue recommending items to any user $u$ in $\ca{V}_t$ (bad users) based on the items liked by other users who belong to $\ca{P}_0(u)$.
\end{rmk}

\noindent\textbf{Theoretical performance guarantee.} In the following, we state our main theoretical result, which is a lower bound on the reward in \eqref{eqn:rewardDef} achieved by Algorithm~\textsc{Collaborative}. To that end, we introduce three natural and prima facie necessary assumptions, which will be used in order to establish our result.
\begin{itemize}
\itemsep 0.2em
    \item[\textbf{A1}]\textbf{No $\Delta$-ambiguous items.} For every user $u \in [\sN]$ and item $i \in [\sM]$, there exists a constant $\Delta\in(0,1/2]$ such that $\left|p_{ui}-1/2\right| \ge \Delta$. 
    \item[\textbf{A2}]\textbf{Minimum portion of likeable items.} There exists a constant $\mu\in[0,1]$, such that for every user $u \in [\sN]$, it holds $\sum_{i=1}^{\sM}\mathds{1}\pp{p_{ui}>1/2}\geq \mu\sM$. 
    
    \item[\textbf{A3}]\textbf{(In)coherence.} There exist constants $\gamma_2\geq\gamma_1\in[0,1)$ such that if two users $u$ and $v$ are of \emph{different} types, then
    $\langle 2\bp_u-\mathbf{1},2\bp_v-\mathbf{1} \rangle \le 4\gamma_1 \Delta^2\sM$, and if they are of the \emph{same} type, then $\langle 2\bp_u-\mathbf{1},2\bp_v-\mathbf{1} \rangle \ge 4\gamma_2 \Delta^2\sM$. Here $\mathbf{1}$ is the all ones vector.
\end{itemize}
In a nutshell, Assumption~\textbf{A1} is necessary to assure that one can infer whether an item is either likable or unlikable with a finite number of samples. The parameter $\Delta$ quantifies the inconsistency (or, noise), where $\Delta=0$ ($\Delta=1/2$) is the completely noisy (noiseless) case. The second condition states that each user likes at least a fraction $\mu$ of the total items. This assumption is made to avoid degenerate situations were a user $u$ does not like any item. Note that one can always ignore those users whose activity is insignificant since their contribution to the reward will be insignificant as well. Evidently, from a practical point of view, any real-world recommendation engine must prioritize users whose activity is significant. Notice that relevant literature, such as \cite{Bresler:2014,Heckel:2017}, makes the same assumptions as well. 

The more interesting assumption is \textbf{A3}. The incoherence part of assumption \textbf{A3} requires that the preference vectors for any two users $u$ and $v$ belonging to different user-types are not too similar, so that the cosine similarity test can separate users of different types over time. The parameter $\gamma_1$ controls this incoherence; the smaller $\gamma_1$ is, the less similar are users of different types. This incoherence assumption appears in \cite{Bresler:2014} as well. The coherence part of assumption \textbf{A3} asks that any two users of the same user-type share a large fraction of the items that they find likable, and this fraction is controlled by the parameter $\gamma_2$. This coherence assumption should be contrasted with \cite{Bresler:2014} where it was assumed that the preference vectors $\bp_u$ and $\bp_v$ of two users $u$ and $v$ from the same user-type to be exactly the same, which is evidently a stronger assumption, and accordingly, our coherence assumption relaxes that significantly. 

\emph{We would like to clarify that the above assumptions are only required for the analysis; Our proposed algorithm can be implemented regardless of these assumptions}. As is shown in Section~\ref{sec:exper}, in real-world applications, our algorithm works well even if these assumptions do not hold. 
We next provide two examples where the typical values of the various parameters in assumptions \textbf{A1}--\textbf{A3} are derived.

\begin{example}
Consider the noiseless case where $\Delta=1/2$. In this case, the users' ratings are deterministic given their user-types. Accordingly we generate $\sK$ $\s{d}$-dimensional binary vectors $\{\mathbf{b}_i\}_{i=1}^{\sK}$ by randomly drawing $\s{d}$ statistically independent $\mathsf{Bernoulli}(1/2)$ random variables, for each user-type. Here $\s{d}\leq\sM$ is some parameter. Then, the preference vector of any user in the $\ell$ user-type (i.e., $\calT_{\ell}$) will be the concatenation of $\mathbf{b}_{\ell}$ with $\sM-\s{d}$ statistically independent $\mathsf{Bernoulli}(1/2)$ random variables. To wit, the preference vector of user $u\in\calT_{\ell}$ is $\bp_u = [\mathbf{b}_{\ell};\mathbf{e}_{u}]$, where $\mathbf{e}_{u}$ is a binary vector whose $\sM-\s{d}$ elements are statistically independent $\mathsf{Bernoulli}(1/2)$ random variables. Now, for any two users $u$ and $v$ from different user-types, it should be clear that the inner product $\frac{1}{\sM}\langle 2\bp_u-\mathbf{1},2\bp_v-\mathbf{1} \rangle$ is merely a sum of $\sM$ Rademacher random variables normalized by $\sM$. Accordingly, a standard concentration inequality on sum of Rademacher random variables tell us that the value of this inner product is in the interval $[-\Theta(\sqrt{\frac{\log\sM}{\sM}}),\Theta(\sqrt{\frac{\log\sM}{\sM}})]$, with probability at least $1-\mathsf{poly}(\sM^{-1})$. Therefore, for the incoherence condition to hold with high probability we need $\gamma_1>\Theta(\sqrt{\frac{\log\sM}{\sM}})$. On the other hand, if $u$ and $v$ are from the same user-type, the inner product of the first $\s{d}$ items is maximal (i.e., unity) by construction. Therefore, using the same arguments it can be shown that the value of the above inner product is at least $\frac{\s{d}}{\sM}-\Theta(\sqrt{\frac{(\sM-\s{d})\log(\sM-\s{d})}{\sM^2}})\geq \frac{\s{d}}{\sM}-\Theta(\sqrt{\frac{\log\sM}{\sM}})$, with high probability. This implies that the coherence condition holds if $\gamma_2\leq\frac{\s{d}}{\sM}-\Theta(\sqrt{\frac{(\sM-\s{d})\log\sM}{\sM^2}})$. When $\s{d}=\sM$, which means that users of the same user-type have exactly the same preference vectors and therefore $\gamma_2$ can get as large as $1$. Otherwise, there is a certain payment depending on how similar the preference vectors are, controlled by $\s{d}$. Finally, the typical value of $\mu$ is clearly around $1/2$ with high probability.  
\end{example}

\begin{example}
We generalize the previous example. Consider the case where each entry of the $\s{d}$-dimensional vectors $\{\mathbf{b}_{\ell}\}_{\ell=1}^{\sK}$ is $\frac{1}{2}+\Delta$ with probability $\mu$ and $\frac{1}{2}-\Delta$ with probability $1-\mu$, for a fixed $\Delta$. Then, as in the previous example, the preference vector of user $u\in\calT_{\ell}$ is $\bp_u = [\mathbf{b}_{\ell};\mathbf{e}_{u}]$, where $\mathbf{e}_{u}$ is now a random vector whose $\sM-\s{d}$ elements are statistically independent, and each element is either $\frac{1}{2}+\Delta$ with probability $\mu$ and $\frac{1}{2}-\Delta$ with probability $1-\mu$. Then, using the same arguments as in the previous example, it can be shown that if users $u$ and $v$ are of different user types, then the incoherence condition holds with high probability when $\gamma_1>(1-2\mu)^2+\Theta(\sqrt{\frac{\log\sM}{\sM}})$. On the other hand, if users $u$ and $v$ are of the same user type, then the coherence condition holds with high probability when $\gamma_2\leq\frac{\s{d}}{\sM}-(1-2\mu)^2-\Theta(\sqrt{\frac{(\sM-\s{d})\log(\sM-\s{d})}{\sM^2}})$.
\end{example}

Let $a\vee b\triangleq\max(a,b)$ and $a\wedge b\triangleq\min(a,b)$, for $a,b\in\mathbb{R}$. We are now in position to state our main result.
\begin{theorem}\label{thm:1}
Let $\delta\in(0,1)$ and $\nu\in(0,1)$ be some pre-specified tolerances. Take as input to \textsc{Collaborative} $\alpha\in(0,4/7]$, any $\lambda\in(\lambda_{-},\lambda_{+})$, and $\Delta_{\mathsf{T}} = \sT\wedge\sqrt{\frac{2\nu\sT}{\mathsf{3V_{2}}}\kappa}$, where $\lambda_{\pm}\triangleq2\gamma_1\Delta^2+\frac{1}{2}\pm\Delta^2(\gamma_2-\gamma_1)$ and $\kappa\triangleq\s{T_{static}}(1-\delta-\mu)$. Define
$\s{T_{static}}\triangleq\frac{2\log(3\sN^2/\delta)}{\Delta^4(\gamma_2-\gamma_1)^2}$ and $\mathsf{T}_{\mathsf{learn}}\triangleq \pp{ 1\vee\frac{3\mathsf{V_{2}}}{2\nu(1-\delta-\mu)}}\mathsf{T}_{\mathsf{static}}.$ Consider the latent
source model and assumptions \textbf{A1}--\textbf{A3}. If at every time-point, the portion of users belonging to any user-type is at least $\nu$, then, for any $\sT_{\s{learn}}\leq\sT\leq\mu\cdot\sM$, the expected proportion of liked items recommended by
\textsc{Collaborative} up until time $\sT$ satisfies
\begin{align}
\mathsf{reward}(\sT)&\geq (1-\delta)\cdot \sT-\kappa\vee\sqrt{\frac{3\s{V_{2}\sT}\kappa}{2\nu}}-2\s{V_{1}}\s{T_{static}}-2\sqrt{\s{V}_1\sT\s{T_{static}}}-\frac{3\s{V_{2}}\sT}{2\nu}\wedge\sqrt{\frac{3\s{V}_{2}\sT\kappa}{2\nu}}.\label{eqn:rewardTh}
\end{align}
For $\sT<\sT_{\s{learn}}$, the reward satisfies $\mathsf{reward}(\sT)\geq\mu\cdot\sT$. 
\end{theorem}
The proof of Theorem~\ref{thm:1} can be found in Appendix~\ref{app:A}, and we now discuss its implications. For $\sT<\sT_{\s{learn}}$, the algorithm may give poor recommendations. This is reasonable since in the first $\sT_{\s{learn}}$ rounds mostly random items are recommended, independently of the users responses, and thus the reward is at least $\mu\cdot\sT$. This is the initial phase for which our CF algorithm gives poor recommendations. Then, for $\sT_{\s{learn}}<\sT<\mu\cdot\sM$, the algorithm becomes efficient. Specifically, when $\mathsf{V_{1}}=\s{V_2}=0$, we get that $\mathsf{reward}(\sT)/\sT \geq (1-\s{T_{learn}}/\sT)\cdot(1-\delta-\eta)$. Therefore, the proposed algorithm becomes near-optimal, as the achieved reward is $(1-\varepsilon')$--close to an oracle that recommends only likeable items and thus achieves a reward of $\sT$. Note that contrary to multi-armed bandit literature, linear reward is common in collaborative filtering frameworks (see, for example, \cite{Bresler:2014,Heckel:2017,Mina2019}). For $\sT>\mu\cdot\sM$, on the other hand, one cannot guarantee that likable items remain, and the learning time (cold-start time) is $\sT_{\s{learn}}=\sT_{\s{static}}$. 

Clearly, when $\s{V_1},\s{V_2}>0$ the reward decreases. Specifically, if both of these parameters scale as $O(\sT^c)$, for some constant $c\in[0,1]$, then the payoff for non-stationarity compared to the static case is of order $O(\sqrt{\sT^{1+c}})$, a sub-linear cost in $\sT$. In particular, ignoring the exact dependency of the reward on the various parameters, the scaling of \eqref{eqn:rewardTh} with $\sT$ is $\mathsf{reward}(\sT)/\sT\geq 1-\delta-O(\sqrt{\sT^{c-1}})$. This result provides a spectrum of orders of the payoff ranging between order $O(\sqrt{\sT})$ (constant number of variations), and of order $O(\sT)$ (number of variations is $O(\sT)$). The sub-linearity growth implies that our user-based CF algorithm has long run average performance that converges to the performance that would have been achieved in the static environment, where users' preferences do not vary. Finally, in terms of the learning time, it can be checked that $\sT_{\s{learn}} = \Theta(\sT^c\log(\sN^2/\delta))$, and thus the condition $\sT>\sT_{\s{learn}}$ boils down to $\sT>\Theta([\log(\sN^2/\delta)]^{\frac{1}{1-c}})$. Therefore, when there are variations, the cold-start time grows, and the scaling of the variations with $\sT$ dictates the poly-log order of this learning phase.

Next, we study the dependency of the learning time in Theorem~\ref{thm:1} on $\gamma_1$ and $\gamma_2$, for Example~1 above. It can be seen that the learning time depend on these parameters via the term $(\gamma_2-\gamma_1)^{-2}$. Accordingly, when $\s{d}=\sM$ in Example~1 we get that $(\gamma_2-\gamma_1)^{-2}$ is of order constant by taking $\gamma_1 = \Theta(\sqrt{\frac{\log\sM}{\sM}})$ and $\gamma_2=1$. In fact, it is evident that this is true when $\s{d}=\Theta(\sM)$ as well, and thus if any two users from the same user-type share a constant fraction of the total number of items that they find likeable, then this has a multiplicative constant effect on the learning time. If however, $\s{d}=o(\sM)$, say, $\s{d} = O(\sM^{q})$, for $q\in(0,1/2)$, then we can set $\gamma_2 = \Theta(\sM^{-c})$, and accordingly, $\sT_{\s{static}}$ will scale as $\sM^{2q}\cdot\log(\sN^2/\delta)$, for $\alpha\to0$. Accordingly, we see that when negligible amount of common items are likeable by users of the same type, then the learning time is significantly larger, as expected. 

Below we mention briefly some of the technical challenges encountered in the proof of Theorem~\ref{thm:1}. First, we establish a connection between the static reward and the reward of a dynamic oracle in the non-stationary setting. This connection is general and can be used in other possible static recommendation system models that incorporate non-stationary environment. One of the main difficulties in the proof of Theorem~\ref{thm:1} is the analysis of how would a variation in the preference of some user affects other users in its estimated neighborhood. Unless this user is detected by Algorithm~\ref{algo:test}, the algorithm does not know the change of this user, and it will keep using this user's feedback to make recommendations for other users. This is one source of added regret incurred by the unsuccessful and incorrect detections of the change-points. Other sources of regret are the cost associated with the detection/testing delay, and a regret incurred by variations happening when testing. These costs are captured by the third and fourth terms at the R.H.S. of \eqref{eqn:rewardTh}. 

\section{Experiments}\label{sec:exper}

We simulate an online recommender system using real-world data in order to understand whether our algorithm performs well, even when the data is not generated by the probabilistic model introduced in Section~\ref{sec:probsetting}. To that end, we follow a similar vein as in \cite{Bresler:2014,Heckel:2017}, and look at movie ratings from the popular Movielens25m dataset,\footnote{https://grouplens.org/datasets/movielens/25m/} which provides 5-star rating and free-text tagging activity from Movielens, a movie recommendation service. We parsed the first $7$ million ratings for our experiment, and consider only those users who have rated at least $225$ movies, ending up with a total number of $\sN=247$ users. 

To avoid any kind of biases, we also restrict ourselves to movies which are more or less equally liked and disliked by the users. To that end, we choose those movies whose average ratings is between $2.5$ and $3.5$, and we found out that $\sM=10149$ such movies exist. Finally, we looked at two genres: \texttt{Action} and \texttt{Romance}. For each user $u \in [\sN]$, we recover piece-wise stationary preferences by the following steps:
\begin{enumerate}
\item We sort the movies rated by user $u$ in ascending order according to the time-stamp.
\item We partition the movies rated by user $u$ into $15$ bins so that each bin contains equal number of movies. We will consider each bin to be a window of time.
\item For each bin, we find the number $\s{a}_u\in\mathbb{N}$ of \texttt{Action} movies rated by user $u$, as well as $\s{r}_u\in\mathbb{N}$ the number of \texttt{Romance} movies rated by the same user.
\end{enumerate}
Accordingly, note that in each bin, the probability of user $u$: liking a movie tagged \texttt{Action} but not \texttt{Romance} is $\mathsf{a}_u/(\mathsf{a}_u+\mathsf{r}_u)$; liking a movie tagged \texttt{Romance} but not \texttt{Action} is $\mathsf{r}_u/(\mathsf{a}_u+\mathsf{r}_u)$; liking a movie tagged both \texttt{Action} and \texttt{Romance} is \texttt{1}, and finally, a movie which does not have any of these tags is \texttt{0}. We want to point out that we consider the number of \texttt{Action} and \texttt{Romance} movies that were \textit{rated} by the user, rather than just \textit{liked}, since any user is biased towards rating the movies he will like (see, \cite{Heckel:2017}), and therefore the number of movies rated by the user is a better indicator of his preference towards the genre. Fig.~\ref{fig:PreferencesVariations0app} shows the probability of $5$ randomly chosen users liking \texttt{Action} movies across $10$ different bins. It is clear that the preferences exhibit a piece-wise stationary behaviour, and that the variations are significant.

We now assume for simplicity that the number of rounds in each bin is $100$ (this value is unknown to the algorithm), and we took the total number of rounds to be $\sT=600$. In lieu of creating the initial disjoint clusters at the beginning of each batch (i.e., $\ca{P}_0$), we recommend $\s{T}_{\s{static}}$ randomly chosen items to all users. For each user $u \in [\sN]$, we take the neighbors of $u$ to be the top $10$ users whose feedback vector has the highest cosine similarity with that of user $u$, over the $\s{T}_{\s{static}}$ recommended items. Further, since $\sT=600$ is quite small, we do not test for bad users in each batch (namely, we skip lines $13-15$ in Algorithm~\ref{algo:user_user_CF_noise}). The reasons for this modification are as follows. First, in the theoretical analysis, we have assumed that ratings of a single \textit{bad user} can potentially result in faulty recommendations for all other users in their user group. However, in practice, that might not be the case as future recommendations are determined by multiple other users who can negate the effect of that \textit{bad} user. Secondly, as the dataset for our experiment is not very large ($10$ neighbors for each user), detecting bad users based on ratings of neighbors can be unreliable. Finally, for small $\s{T}$, $\s{T}_{\s{static}}$ is comparatively large and therefore testing for \textit{bad users} can potentially bias the accumulated reward towards larger batch-sizes. Nevertheless, as we will show, our experiment clearly demonstrates the dependence on $\Delta_{\sT}$ and $\s{T}_{\s{static}}$ in the non-stationary setting. We run Algorithm~\ref{algo:em} with $\s{T}_{\s{static}}=10$ and $p_{\s{R}}=0.1$, for several different values of the batch-size $\Delta_{\s{T}}$, each for $5$ different iterations. The performance of the algorithms is measured in terms of the average cumulative reward up to time $\sT$, namely, 
$$
\mathsf{acc}\text{-}\mathsf{reward}(\sT)\triangleq\sum_{t\in[\sT]}\frac{1}{\sN}\sum_{u\in[\sN]}\mathbf{R}_{u\pi_{u,t}},
$$ 
where $\pi_{u,t}$ is the item recommended by the algorithm to user $u$ at time $t$.
The average cumulative reward up to time $\sT$ is given in Table~\ref{Table:10}. From this table, it is clear that the highest average cumulative reward is obtained when the batch-size is $\Delta_{\s{T}}=100$, and decreases gradually as the batch-size increases. Finally, not that since we are not detecting \textit{bad users} in our experiments, the knowledge of $\s{V}_1$ is not required ($\s{V}_1$ is only used to set $p_{\s{T}}$). Notice that $\s{V}_2$ is used to set the batch-size $\Delta_{\s{T}}$ correctly. Since an incorrect value of $\s{V}_2$ results in a sub-optimal value for $\Delta_{\s{T}}$, computing the average cumulative reward by iterating through different values of $\Delta_{\s{T}}$ also gives an idea about the sensitivity of our algorithms with respect to this mis-specification. As can be seen from our results, the highest value of $\mathsf{acc}\text{-}\mathsf{reward}(\sT)$ was achieved when $\Delta_{\s{T}}=100$, while the $\mathsf{acc}\text{-}\mathsf{reward}(\sT)$ degrades gracefully with the mis-specification of $\Delta_{\s{T}}$ (or, $\s{V}_2$).

\begin{table}[h!]
\centering
\begin{tabular}{||c | c||} 
\hline
 $\Delta_{\s{T}}$ & $\mathsf{acc}\text{-}\mathsf{reward}(\sT)$ \\ [0.5ex] 
 \hline\hline
 50 & 316.707\\
  \hline
 100 & 325.716\\
  \hline
 150 & 306.538\\
  \hline
 200 & 278.219\\
  \hline
 300 & 278.642\\
  \hline
 350 & 224.893\\
  \hline
 400 & 239.410\\
  \hline
 450 & 204.127\\
  \hline
 500 & 162.96\\
  \hline
 550 & 169.97\\
  \hline
 600 & 137.40 \\
 \hline
\end{tabular}
\caption{Accumulated reward as a function of the batch-size: $\Delta_{\s{T}}=600$ corresponds to the static case, and $\Delta_{\s{T}}=100$ corresponds to the optimal value.}
\label{Table:10}
\end{table}

Next, we illustrate the benefit of our algorithm compared to the static algorithm even in a stationary environment. To that end, we run Algorithm~\ref{algo:em} with $\Delta_{\s{T}}\in\{100,600\}$, $\sT_{\s{static}}\in\{10,30,60,80,100\}$, and assume a single bin of size $\sT=600$. Our results are presented in Fig.~\ref{fig:Second_expapp}, and perhaps surprisingly, Algorithm~\ref{algo:em} with $\Delta_{\s{T}}=100$ achieves a better accumulated reward compared to $\Delta_{\s{T}}=600$ (static algorithm), for small values of $\sT_{\s{static}}$. The main reason for this phenomenon is because for Algorithm~\ref{algo:em} with $\Delta_{\s{T}}=600$, the neighbors of any user might not be well chosen due to small values of $\s{T}_{\s{static}}$ because of which the user will receive poor recommendations throughout the entire time frame. On the other hand, running Algorithm~\ref{algo:em} with $\Delta_{\s{T}}=100$ restarts Algorithm \ref{algo:user_user_CF_noise} at periodic intervals. As a result, the users have a good set of neighbors in some batches and a bad set in others, but the cumulative reward concentrate because the neighbors are independent across the batches. However, the performance of the algorithm with $\Delta_{\s{T}}=600$ improves as $\sT_{\s{static}}$ gets larger since the quality of the estimated neighborhood improves. This experiment hits that it is better to restart the recommendation algorithm periodically, i.e., follow Algorithm~\ref{algo:em} (with $\Delta_{\sT}<\sT$) even in stationary environments. We would like to emphasize that an insufficient number of samples for the initial clustering, results in a worse accumulated reward for $\Delta_{\s{T}}=600$. In practice, however, the number of samples used for the initial clustering might be difficult to determine a-priori. In that situation, we suggest to restart the algorithm periodically with a small value of $\s{T}_{\s{static}}$. Indeed, since the batches are independent, the accumulated reward concentrates due to the law of large numbers. 

\begin{figure}[t]
    
\centering
\includegraphics[width=0.5\textwidth]{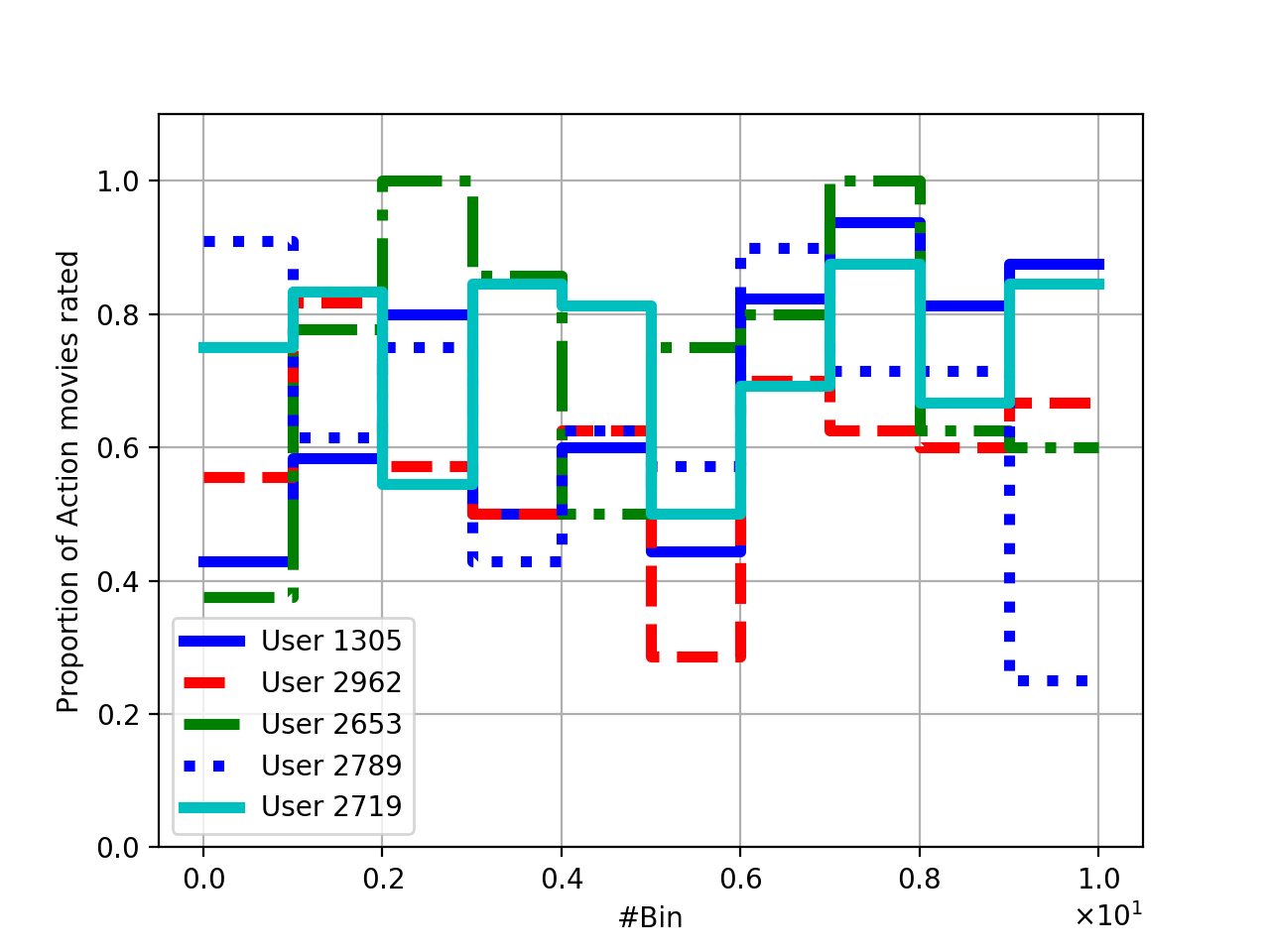}
\caption{The probability $\s{a}_u/(\s{a}_u+\s{r}_u)$ of user $u$ liking a movie with \texttt{Action} tag but not \texttt{Romance} tag, for five different users, across $10$ different bins/windows.}
\label{fig:PreferencesVariations0app}
\end{figure}

\begin{figure}[t!]
    
\centering
\includegraphics[width=0.5\textwidth]{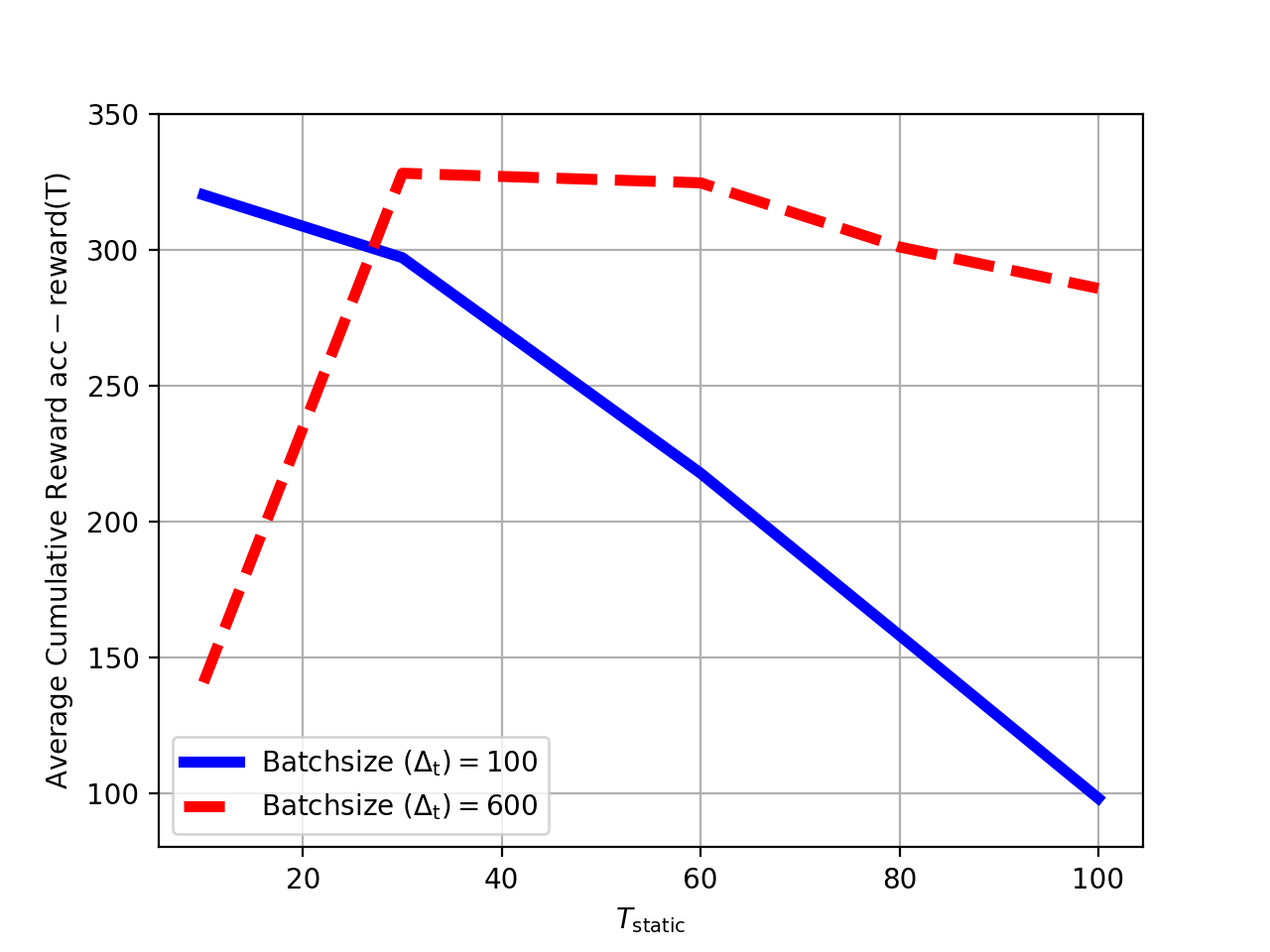}
\caption{Comparison of Average Cumulative Reward $\s{acc-reward(T)}$ for batchsize $(\Delta_{\sT})\in\{100,600\}$ and $\sT_{\s{static}}\in\{10,30,60,80,100\}$.}
\label{fig:Second_expapp}
\end{figure}

Next, we further compare the performance of our algorithm to the static case \cite{Bresler:2014}, and to the Popularity Amongst Friends (PAF) algorithm \cite{Dabeer12}. We consider the same setting as in \cite{Bresler:2014}. In particular, we again quantize movie ratings $\geq4$ as $+1$ (likable), movie ratings $<3$ as $-1$ (unlikable), and missing ratings as $0$. We consider the top $\sN = 250$ and $\sM=500$ users and movies, respectively. This results in $\approx80\%$ nonzero entries among the total number of entries in the rating matrix. There are of course missing entries in the resulted rating matrix. Accordingly, in our simulation if at a certain time, item $i$ was recommended to user $u$, who has not rated that item, we receive $0$ reward. Despite that we will still treat item $i$ as being consumed by user $u$, and accordingly, item $i$ cannot be recommended to user $u$ again. 
Since we allow algorithms to recommend an item to a given user only once, after $\sT=\sM = 500$ time steps, all items have been recommend to all users. As before, the performance of the algorithms is measured in terms of the average cumulative reward up to time $\sT$. 

\begin{figure}[t!]
\begin{center}
\includegraphics[width=0.5\textwidth]{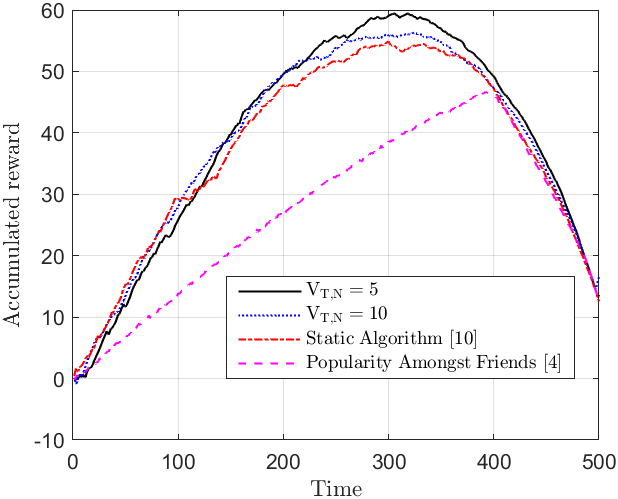}
\end{center}
\caption{The accumulated reward over time achieved by Algorithm~\textsc{Collaborative} and existing recommendation algorithm Popularity Amongst Friends \cite{Dabeer12}, for several values of the variation budget $\s{V}\in\{0,5,10\}$, using Movielens10m dataset.}
\label{fig:RewardVsVariation}
\end{figure}

In the simulation, we run Algorithm~\textsc{Collaborative} with three different values for the variation budget $\s{V}=\s{V_1}=\s{V_2}\in\{0,5,10\}$, and recall that $\s{V}=0$ corresponds to the static case \cite{Bresler:2014}. The results are given in Fig.~\ref{fig:RewardVsVariation}. It is evident that Algorithm~\textsc{Collaborative} significantly outperforms PAF algorithm, a fact which was already observed in \cite{Bresler:2014}. More importantly we see that assuming that $\s{V}=5$, and accordingly recommending in batches, gives the best results among the other values of $\s{V}$, and in particular the static case. Except for coping with variations in the preferences of users, this can be attributed also to \emph{model mismatch}. To wit, the static algorithm \textsc{Recommend} was designed for a certain probabilistic model which may not capture certain phenomena in real-world datasets. Accordingly, it might be the case that the algorithm will ``stuck" on a certain wrong rating trajectory which will hinder the rate at which likeable movies are recommended. Working in batches, and by which letting the algorithm to ``restart" occasionally, may compensate for this mismatch. Finally, note that the reason for the $\cap$-shape of the obtained curves is the fact that after recommending most of the likable items (around $t\approx310$), mostly unlikable movies are left to recommend, until we exhaust all possible movies. 

\section{Conclusion and Outlook}\label{sec:concout}
In this paper, we introduced a novel model for online non-stationary recommendation systems, where users may change their preferences over time adversarially. For this model, we analyzed the performance of a CF recommendation algorithm, and derived a lower bound on its achievable reward. 

We hope our work has opened more doors than it closes. Apart from tightening the obtained lower bound on the reward, there are several exciting directions for future work. First, it is of significant importance to tackle the case where the number of variations is unknown. Devising universal algorithms which are oblivious to the  knowledge of the non-stationarity, and proving theoretical guarantees is quite challenging (see, for example, the recent papers \cite{Karnin16,Auer19a,Chen2019ANA} where the problem of non-stationary MAB with unknown number of variations was considered). Secondly, it is very interesting and technically challenging to derive information-theoretic upper bounds on the performance (reward) of any CF algorithm for the general model introduced in this paper. The results of this paper can be rather directly generalized to one-class recommendation systems where users only rate what they like and never reveal what they dislike. It would be interesting to introduce and analyze models which combines both content/graph information on top of the collaborative filtering information. Also, while in this paper our ultimate goal was to design recommender system which maximize the number of likes, in some applications one might want to take into account other aspects, such as fairness, novelty and multi-stakeholder recommender systems. Formally analyzing such aspects has not been done, and is of practical and theoretical importance. Finally, as was mentioned in the Appendix~\ref{sec:exper} there are several inherent challenges with standard CF datasets used for simulating (non-stationary) online recommender systems. Implementing a real interactive online recommendation system and testing our algorithms over it is an important step towards a complete understanding of CF based recommender systems.

\bibliographystyle{abbrv}
\bibliography{bibfile}

\begin{appendices}

\section{Proof of Theorem~\ref{thm:1}}\label{app:A}
To prove Theorem~\ref{thm:1}, we establish first a few accompanying results. We start with the following lemma which bounds the probability that user of different (same) type have the same response. For this lemma, we assume that users \emph{cannot} change their type over time, and denote the type of user $u\in[\sN]$ by $\calT_u$.
\begin{lemma}[Same Response Lemma]\label{lem:same}
Consider the latent
source model and the incoherence Assumption \textbf{A3}. Let $\ell$ be an item chosen uniformly at random from $[\sM]$. Then, the probability that two users $u$ and $v$ rate $\ell$ in the same way is:
\begin{align}
    \pr\pp{\mathbf{R}_{u\ell}=\mathbf{R}_{v\ell}\vert\calT_u \neq \calT_v}\leq 2\gamma_1\Delta^2+\frac{1}{2},
\end{align}
for users of different types, and,
\begin{align}
    \pr\pp{\mathbf{R}_{u\ell}=\mathbf{R}_{v\ell}\vert\calT_u = \calT_v}\geq2\gamma_2\Delta^2+\frac12,
\end{align}
for users of the same type.
\end{lemma}
\begin{proof}
Notice that for two users $u$ and $v$ belonging to different user groups, the probability in question is
\begin{align*}
\pr\pp{\mathbf{R}_{u\ell}=\mathbf{R}_{v\ell}\vert\calT_u \neq \calT_v}&=\frac{1}{\sM}\sum_{i=1}^{\sM}\pp{p_{u,i}p_{v,i}+(1-p_{u,i})(1-p_{v,i})} \\
&\quad=\frac{1}{\sM}\sum_{i=1}^{\sM}\pp{ \frac{(2p_{u,i}-1)(2p_{v,i}-1)}{2}+\frac{1}{2}}\\
&\quad=\frac{1}{\sM} \langle 2\bp_u-\f{1},2\bp_v-\f{1} \rangle+\frac{1}{2} \\
&\quad\le 2\gamma_1\Delta^2+\frac{1}{2},
\end{align*}
where the inequality follows from the incoherence Assumption~\textbf{A3}. Similarly, for two users of the same type,
\begin{align*}
\pr\pp{\mathbf{R}_{u\ell}=\mathbf{R}_{v\ell}\vert\calT_u = \calT_v}&=\frac{1}{\sM} \langle 2\bp_u-\f{1},2\bp_v-\f{1} \rangle+\frac{1}{2} \\
&\ge 2\gamma_2\Delta^2+\frac12,
\end{align*}
where, again, the inequality follows from the coherence Assumption~\textbf{A3}.
\end{proof}
The following lemma gives a condition on the number of random recommendations needed for the cosine-similarity test to output the correct clustering with high probability, assuming that no variations happened during the test. We establish a few notations. Let $\calT_{\s{test}}\subseteq[\sM]$ be a set of $\s{L}$ items chosen uniformly at random from $\sM$. Let $\s{Y}_{u,v}\in\{0,1\}$ be a binary variable indicating whether $(u,v)$ are in the same cluster or not, for $u,v\in[\sN]$. Using the responses $\{\mathbf{R}_{u,i}\}_{u\in[\sN],i\in\calT_{\s{test}}}$, we would like to infer the values of $\s{Y}_{u,v}$ for all $u,v\in[\sN]$. For any pair of distinct users $u,v\in[\sN]$, let $\s{X}_{u,v}$ be the random variable corresponding to the number of items for which $u$ and $v$ had the same responses. Finally, we let
\begin{align}
\hat{\s{Y}}_{u,v} = \begin{cases}
1,\ &\mathrm{if}\;\s{X}_{u,v}\geq\lambda\cdot\s{L}\\
0,\ &\mathrm{otherwise},
\end{cases}\label{eqn:Test}
\end{align}
for some $\lambda\geq0$. We have the following result.
\begin{lemma}{\label{lem:test}}
Consider the latent source model and the incoherence Assumption \textbf{A3}. Let $\delta\in(0,1)$. For any $\s{L}\geq\frac{2\log(3\sN^2/\delta)}{\Delta^4(\gamma_2-\gamma_1)^2}\triangleq\s{T_{static}}$, and any $\lambda\in[\lambda_{-},\lambda_{+}]$ with $\lambda_{-} = 2\gamma_1\Delta^2+\frac12+\sqrt{\frac{2}{\s{L}}\log(3\sN^2/\delta)}$ and $\lambda_{+} = 2\gamma_2\Delta^2+\frac12-\sqrt{\frac{2}{\s{L}}\log(3\sN^2/\delta)}$, the test in \eqref{eqn:Test} discriminates between $\s{Y}_{u,v}=0$ and $\s{Y}_{u,v}=1$, for any pair of users $u,v\in[\sN]$, with probability at least $1-\delta/3$.
\end{lemma}
\begin{proof}
First, it is clear that Lemma~\ref{lem:same} implies that
\begin{align*}
    &\bE\left[\s{X}_{u,v} \vert \s{Y}_{u,v} = 0\right] \le \s{L}\Big(2\gamma_1\Delta^2+\frac12\Big),\\
    &\bE\left[\s{X}_{u,v} \vert \s{Y}_{u,v} = 1\right] \ge \s{L}\Big(2\gamma_2\Delta^2+\frac12\Big).
\end{align*}
Then, we note that $\s{X}_{u,v}$ is a sum of $\s{L}$ random variables in $[-1,1]$, drawn without replacement from $[\sM]$. Accordingly, Hoeffding's inequality gives,
\begin{align}
\pr\pp{\left.\s{X}_{u,v}\geq\lambda_{-}\cdot\s{L}\right|\s{Y}_{u,v} = 0}&\leq \exp\pp{-\frac{\p{\lambda_{-}\cdot\s{L}-\bE\left[\s{X}_{u,v} \vert \s{Y}_{u,v} = 0\right]}^2}{2\s{L}}}\\
&\leq \exp\pp{-\frac{\p{\lambda_{-}-2\gamma_1\Delta^2-\frac12}^2}{2}\s{L}},
\end{align}
and
\begin{align}
\pr\pp{\left.\s{X}_{u,v}\leq\lambda_{+}\cdot\s{L}\right|\s{Y}_{u,v} = 1}&\leq\exp\pp{-\frac{\p{2\gamma_2\Delta^2+\frac12-\lambda_{+}}^2}{2}\s{L}}.
\end{align}
Therefore, taking $\lambda_{-} = 2\gamma_1\Delta^2+\frac12+\sqrt{\frac{2}{\s{L}}\log(3\sN^2/\delta)}$ and $\lambda_{+} = 2\gamma_2\Delta^2+\frac12-\sqrt{\frac{2}{\s{L}}\log(3\sN^2/\delta)}$, we obtain that
\begin{align}
\pr\pp{\left.\s{X}_{u,v}\geq\lambda_{-}\cdot\s{L}\right|\s{Y}_{u,v} = 0}&\leq\frac{\delta}{3\sN^2},\label{eqn:highProb1}
\end{align}
and
\begin{align}
\pr\pp{\left.\s{X}_{u,v}\leq\lambda_{+}\cdot\s{L}\right|\s{Y}_{u,v} = 1}&\leq\frac{\delta}{3\sN^2}.\label{eqn:highProb2}
\end{align}
Picking any $\lambda\in[\lambda_{-},\lambda_{+}]$, we can see that the bounds in \eqref{eqn:highProb1}--\eqref{eqn:highProb2}, with $\lambda_{-}$ and $\lambda_{+}$ replaced by $\lambda$. This is equivalent to $\pr\pp{\left.\hat{\s{Y}}_{u,v}\neq\s{Y}_{u,v}\right|\s{Y}_{u,v} = \ell}\leq\delta/(3\sN^2)$, for $\ell=0,1$. Such $\lambda$ exists if $\lambda_{+}\geq\lambda_{-}$, which holds whenever,
\begin{align}
\s{L}\geq \frac{2\log(3\sN^2/\delta)}{\Delta^4(\gamma_2-\gamma_1)^2}=\s{T_{static}}.
\end{align}
Finally, taking a union bound over all pairs of users (we trivially have at most $\s{N}^2$ such pairs) we conclude that we can correctly infer the values of $\s{Y}_{u,v}$ for all $u,v\in[\sN]$ (and therefore cluster all such pairs of users correctly), with probability at least $1-\delta/3$, as claimed.
\end{proof}

We would like to mention here that the test described above can only distinguish between whether $\s{Y}_{u,v}=1$ or $\s{Y}_{u,v}=0$, \emph{assuming} that users did not change their type during the test. If, however, a test is conducted when there are switches, we can still infer the clustering of those users who have not changed during the test correctly.

We are now in a position to prove Theorem~\ref{thm:1}. With some abuse of notation, let us denote by $\mathsf{reward}(\calB_\ell)$ the expected reward accumulated in batch $\calB_\ell$, i.e.,
\begin{align}
    \mathsf{reward}(\calB_\ell)&\triangleq\bE \left[ \sum_{t \in \calB_\ell}\frac{1}{\sN}\sum_{u=1}^{\sN}\mathds{1}[\mathbf{R}_{u\pi_{u,t}}=1]\right]\nonumber \\
    &=|\calB_\ell|-\bE \left[ \sum_{t \in \calB_\ell}\frac{1}{\sN}\sum_{u=1}^{\sN}\mathds{1}[\mathbf{R}_{u\pi_{u,t}}=0]\right]\nonumber\\
    &\triangleq|\calB_\ell|-\mathsf{regret}(\calB_\ell),
\end{align}
where $\mathsf{regret}(\calB_\ell)$ is the regret accumulated during batch $\calB_\ell$.
As can be seen from Algorithm~\textsc{Collaborative}, we decompose the recommendation horizon $\sT$ to a sequence of batches of size $\Delta_{\mathsf{T}}$ each. To obtain Theorem~\ref{thm:1}, we will relate the total reward/regret with the local reward/regret of the static algorithm \textsc{Recommend}. Specifically, let $\calB_\ell$, for $\ell=1,2,\ldots,\ceil{\sT/\Delta_{\mathsf{T}}}$, denote the $\ell$'th batch of size $\Delta_{\mathsf{T}}$, and let $t_{\calB_\ell}$ be the ending time of batch $\calB_\ell$. 
We will keep track of a set of users $\ca{V}_t \subseteq [\sN]$ which will include all those users for whom we have been able to identify that they have have changed their user groups at some point of time during the batch $\ca{B}_{\ell}$. We initialize $\ca{V}_{t_{\ca{B}_{\ell-1}}+1} = \phi$ at the beginning of the batch to be the empty set. We define
\begin{align}
    \s{V}_{\ca{B}_\ell,1}\triangleq\sum_{t \in \calB_\ell\setminus t_{\calB_\ell}} \mathds{1}\left[\calT_{u}(t) \neq \calT_{u}(t+1),\;\text{for some}\;u\in[\sN]\right], \label{eqn:Vbatchdef}
\end{align}
as the number of variations that have occurred during the batch $\ca{B}_{\ell}$. Furthermore, we let
\begin{align}
    \s{V}_{\ca{B}_\ell,2}\triangleq\frac{1}{\sN}\sum_{u \in [N]}\sum_{t \in \calB_\ell\setminus t_{\calB_\ell}} \mathds{1}\left[\calT_{u}(t) \neq \calT_{u}(t+1)\right], \label{eqn:Vbatchde2}
\end{align}
as the total number of variations that have occurred during the batch $\ca{B}_{\ell}$. For $\tau \in \ca{B}_{\ell}$, we define $\s{Z}_{\tau}$ to be an indicator random variable which is unity if some user switches its type in a window of $2\cdot\s{T_{static}}$ around round $\tau$ within the batch $\ca{B}_{\ell}$. For $\tau\in\ca{B}_{\ell}$, let us denote $\s{W}_{\tau}:=\{\max\{\tau-\s{T_{static}},t_{\calB_{\ell-1}+1}\},\dots,\min\{\tau+\s{T_{static}},t_{\calB_{\ell}}\}\}$ as the window of size $2\cdot\s{T_{static}}$ around $\tau$. Then, note that $\s{Z}_{\tau}$ can be written as
\begin{align}
    \s{Z}_{\tau}=\mathds{1}\left[\sum_{u\in[\sN]}\sum_{ t \in \s{W}_{\tau}}\mathds{1}\left[\calT_{u}(t) \neq \calT_{u}(t+1) \right]>0 \right].\label{eqn:Zdef}
\end{align}

As can be seen from Algorithm~\textsc{Recommend}, at every round in each batch, we start a test with probability $1/\sqrt{\Delta_{\sT}}$, which involves recommending randomly sampled items to every user for $\s{T_{static}}$ rounds. After each such test, we can use Lemma~\ref{lem:test} to partition the set of users. In addition, in the fourth step of Algorithm~\textsc{Recommend} we conduct a \emph{reference test} at the beginning of the batch. In the sequel, we denote this test by $\s{Test}_0$, and further denote the $(j+1)^{\s{th}}$ test by $\s{Test}_j$. The partition induced by the $(j+1)^{\s{th}}$ test is denoted by $\ca{P}_{\s{Test}_j}$. By comparing the partitions $\ca{P}_{\s{Test}_j}$ and $\ca{P}_{\s{Test}_0}$, we will be able to partially identify users who have changed their user groups in the batch. This is done in Algorithm~\textsc{Test}. We will call those users who have changed their user groups in a particular batch as \textit{bad} users and those users who have not changed their user groups throughout the batch as \textit{good} users. Moreover, a user is also \textit{good} until he changes his user group and will be denoted as \textit{bad} from the round he changes his group. In order to bound the regret over each batch we will consider the following three cases:
\begin{itemize}
    \item \underline{\textbf{Case 1:}} Consider the situation where at least $2/3$ of the users of any particular user group have changed their user group. We denote this event by $\ca{E}_1$. In such a case, we will upper bound the regret in the batch $\ca{B}_{\ell}$ by $\s{\Delta}_T$. Notice that, since $\s{V}_{\ca{B}_{\ell},2} \ge \frac{2\nu}{3}$, therefore conditioned on $\calE_1$, we have $\s{regret}(\ca{B}_{\ell}) \le \frac{3\Delta_{\s{T}}\s{V}_{\ca{B}_{\ell},2}}{2\nu}$. 
    \item \underline{\textbf{Case 2:}} In this case, we will assume that for every user group, at most $1/3$ of the users change their user groups in the batch. For any test $\ca{P}_{\s{Test}_j}$, notice that we can actually end up with more than $\s{K}$ clusters (say we have $\s{K}'$ clusters) because of variations. In that case, we will identify all the users in the smallest $\s{K}'-\s{K}$ clusters as users who have changed their user group. Note that, it is possible that we make a mistake in this process because one of the clusters in the smallest $\s{K}'-\s{K}$ clusters might correspond to good users who have not changed their user group. This, however, must mean that a larger cluster among the largest $\s{K}$ clusters must correspond to users who have changed. This in turn implies that at least $\frac{2\nu \s{N}}{3}$ users have changed since the size of the smaller cluster corresponding to users who do not change their group throughout the batch $\ca{B}_{\ell}$ is at least $\frac{2\nu \sN}{3}$. We will denote this event by $\ca{E}_2$. As in the previous case, we trivially upper bound the regret in the batch $\ca{B}_{\ell}$ by $\Delta_{\s{T}}$, and similarly to Case 1, we have $\s{regret}(\ca{B}_{\ell})\le \frac{3\Delta_{\s{T}}\s{V}_{\ca{B}_{\ell},2}}{2\nu}$, conditioned on $\calE_2$. 
    \item \underline{\textbf{Case 3:}} In this case, as in the previous case, we assume that for every user group, at most $1/3$ of the users change their user groups in the batch. Contrary to the previous case, we will also assume that in every test with more than $\s{K}$ clusters (say, $\s{K}'$ clusters), the users in the smallest $\s{K}'-\s{K}$ users correspond to users who have changed their user group. For a future test $j$ started at round $\tau_{\s{test},j}$ such that $\s{Z}_{\tau_{\s{test},j},u}=0$, we will compare the partitions $\ca{P}_{\s{Test}_0}$ and $\ca{P}_{\s{Test}_j}$ by establishing a bijective mapping between the clusters of the two partitions. For every cluster $\ca{C}$ in $\ca{P}_{\s{Test}_0}$, we can find a cluster $\ca{C}'$ in $\ca{P}_{\s{Test}_j}$ such that at least two-thirds of the elements in $\ca{C},\ca{C'}$ are common. Subsequently, for all those users in $\ca{C}$ which are not present in $\ca{C}'$, we correctly identify them as users who have changed their user groups. For a pair of distinct users $(u,v) \subset [\sN]\times [\sN]$ belonging to the same user group at the beginning of the batch $\ca{B}_{\ell}$, we call them \textit{interesting} if one of them have changed their user group. Note that, for any pair of interesting users $(u,v)$ where one of them have changed their user group at any round after the reference test is conducted and remains in different user group before the $(j+1)^{\s{th}}$ test is started will belong to different clusters in $\ca{P}_{\s{Test}_j}$. Note that it is possible that $\s{Z}_{t_{\ca{B}_{\ell-1}}+1,u}=1$, i.e., some user $u$ might change their user group during the first $\s{T_{static}}$ rounds when the reference test is being conducted. Since we can label the top $\s{K}$ clusters (by the corresponding user-group) returned by the reference test as we know that two-thirds users of every user group did not change. We denote by $\ca{P}_{\s{Test}_0}(u)$ the cluster (label) $u$ belongs to in the partition returned by the reference test $\s{Test}_0$. Let us define an indicator random variable $\s{L}_u$ which is unity if user $u$ has changed his user group in the first $\s{T_{static}}$ rounds. Consider such a user $u$ for which $\s{L}_u=1$. In that case, three things are possible at the end of the reference test:
    \begin{enumerate}
        \item $u$ might belong to the smallest $\s{K}'-\s{K}$ clusters in the reference test in which case $u$ is identified as a user who has changed his user group and he is not involved in the main algorithm started after the reference test, i.e., $u$ is added to the set $\ca{V}_{\s{T_{static}}}$. We define an indicator random variable $\s{X}_{u,1}$ which is unity if user $u$ has changed his user group during the first $\s{T_{static}}$ rounds in the batch, and is returned in the smallest $\s{K}'-\s{K}$ clusters at the end of the reference test.
        \item $u$ belongs to the cluster corresponding to his new user group in which case we will not be able to infer that $u$ has changed his user group. In this case, we will consider $u$ to be a \textit{good} user unless he changes his user group later. We will consider his user group at the end of the reference test ($\ca{P}_{\s{Test}_0}(u)$ which is same as $\ca{T}_u(t_{\ca{B}_{\ell-1}}+1+\s{T_{static}})$) to be his actual user group. We will call this a \textit{special case} and we define an indicator random variable $\s{X}_{u,2}\triangleq\mathds{1}[\ca{P}_{\s{test}_0}(u) =\ca{T}_u(t_{\ca{B}_{\ell-1}}+1+\s{T_{static}})]$ which is unity if user $u$ changes his user group during the first $\s{T_{static}}$ rounds in the batch, and belongs to his final user group (the user group he belongs to at the end of the reference test).
        \item $u$ remains in his original user group (or an intermediate user group if he changes his user group multiple times during the reference test). We define an indicator random variable $\s{X}_{u,3}$ which is unity if the user changes his user group in the first $\s{T_{static}}$ rounds and does not belong to his final user group (the user group he belongs to at the end of the reference test) at the end of the reference test, i.e., $\s{X}_{u,3}\triangleq\mathds{1}[\ca{P}_{\s{test}_0}(u) \neq\ca{T}_u(t_{\ca{B}_{\ell-1}}+1+\s{T_{static}})]$. For a round $\tau>\s{T_{static}}$ in $\calB_\ell$, we will define an indicator random variable $\s{J}_{u,\tau}=\mathds{1}[\ca{T}_u(\tau)\neq \ca{P}_{\s{Test}_0}(u)]$, which is unity if user $u$ is in a different group at round $\tau$ than the user group of $u$ that was returned by the reference test. 
    \end{enumerate}
\end{itemize}
We are now in a position to bound the regret over each batch. To that end, we will decompose the regret into a few terms and analyze the contribution of each term separately. First, as we described above conditioned on Cases 1 and 2, namely, $\calA\triangleq\calE_1\cup\calE_2$ we have
\begin{align}
\mathsf{regret}(\calB_\ell\vert\calA)&\triangleq \frac{1}{\sN}\sum_{t \in \ca{B}_{\ell}\setminus\calT_{\s{test,\ell}}} \sum_{u \in [\sN]} \bb{E} \left[\mathds{1}\left[\mathbf{R}_{u,\pi_{u,t}} = 0 \mid \ca{A} \right] \right]\\
&\leq \frac{3\Delta_{\sT}\s{V}_{\ca{B}_{\ell},2}}{2\nu }.\label{eqn:A0}
\end{align}
Next, we analyze Case~3, where we condition on $\calA^c$, namely,
\begin{align}
\mathsf{regret}(\calB_\ell\vert\calA^c) \triangleq \frac{1}{\sN}\sum_{t \in \ca{B}_{\ell}\setminus\calT_{\s{test,\ell}}} \sum_{u \in [\sN]} \bb{E} \left[\mathds{1}\left[\mathbf{R}_{u,\pi_{u,t}} = 0 \mid \ca{A}^c \right] \right].
\end{align}
We do that by considering each of the sub-cases listed above.
\subsection{Variations When Testing}
We bound the regret for those rounds in the batch for which $\s{Z}_{\tau}=1$. Specifically, for a round $\tau\in\ca{B}_{\ell}$, we denote the event $\ca{E}_{\tau,1}$ when $\s{Z}_{\tau}=1$, which by definition imply that there is a variation in a window of size $2\cdot\s{T_{static}}$ around round $\tau$ for some user. In particular, using the definitions in \eqref{eqn:Vbatchdef} and \eqref{eqn:Zdef}, we note that
\begin{align}
\sum_{\tau \in \ca{B}_{\ell}\setminus\calT_{\s{test,\ell}}}\s{Z}_{\tau}&\leq \sum_{\tau \in \ca{B}_{\ell}}\sum_{ t \in \s{W}_{\tau}}\mathds{1}\left[\calT_{u}(t) \neq \calT_{u}(t+1) ,\;\text{for some }u\in[\sN]\right]\\
&\leq 2\cdot\s{V}_{\ca{B}_{\ell},1}\cdot\s{T_{static}}.
\end{align}
Therefore, we can bound the regret in those rounds and users where $\s{Z}_{\tau}=1$ by 
\begin{align}
\s{A}_2&\triangleq\frac{1}{\sN}\bb{E} \left[\sum_{t \in \ca{B}_{\ell}\setminus\calT_{\s{test,\ell}}} \sum_{u \in [\sN]:\s{Z}_{t} = 1} \mathds{1}\left[\mathbf{R}_{u,\pi_{u,t}} = 0 \mid \ca{A}^c \right] \right] \\ 
& \le \bb{E} \left[\sum_{\tau \in \ca{B}_{\ell}} \mathds{1}\left[\s{Z}_{t}=1\right] \right]\\
&\le 2\cdot\s{V}_{\ca{B}_{\ell},1}\cdot\s{T_{static}}.\label{eqn:A1}
\end{align}
\subsection{Regret Due To Testing}
We bound the regret for those rounds where we test in Algorithm~\textsc{Recommend}. Specifically, for a round $\tau \in \ca{B}_{\ell}$, we define the indicator random variable $\s{Y}_{\tau}$ which is unity when a test is being conducted at the round $\tau$. We have
\begin{align}
\s{A}_3&\triangleq \frac{1}{\sN}\bb{E} \left[\sum_{t \in \ca{B}_{\ell}\setminus\calT_{\s{test,\ell}}: \s{Y}_{\tau} = 1} \sum_{u \in [\sN]} \mathds{1}\left[\mathbf{R}_{u,\pi_{u,t}} = 0 \mid \ca{A}^c \right] \right] \\ 
& \le \bb{E}\left[\sum_{t\in \ca{B}_{\ell}} \mathds{1}\left[\s{Y}_{\tau}=1\right] \right]\\
&\le \Delta_{\s{T}}\cdot p\cdot\s{T_{static}},\label{eqn:A2}
\end{align}
where we have used the fact that $\pr[\s{Y}_{\tau}=1] = \pr[\tau\in\mathrm{Test}] = p$, $|\calB_\ell|=\Delta_{\sT}$, and each test takes $\s{T_{static}}$ rounds.
\subsection{Undetected Bad Users}
For a user $u \in [\sN]$, we define an indicator random variable $\s{B}_{u,t}$ which is unity if the user is not included in the set of bad users $\ca{V}_t$ at round $t \in \ca{B}_{\ell}$. Furthermore, for a round $t$ after the reference test, namely, $t\in\calB_\ell\setminus\calT_{\s{test},\ell}$, where
$\calT_{\s{test},\ell}\triangleq[t_{\ca{B}_{\ell-1}}+1,\ldots,t_{\ca{B}_{\ell-1}}+\s{T_{static}}]$, define an indicator random variable $\s{H}_t$ which is unity if there is a \textit{bad} user which is undetected (or, untested) involved in the algorithm. As we explain Below this random variable can be decomposed into the union of three sub-cases which we discussed above. For any round $t\in\calB_\ell\setminus\calT_{\s{test},\ell}$, we have:
\begin{itemize}
\item A user $u$ who satisfies $\s{L}_u=1,\s{B}_{u,t}=1,\s{X}_{u,3}=1,\s{J}_{u,t}=1$ and $\s{Z}_{t} = 0$, is one who has changed his user group in the first $\s{T_{static}}$ rounds in the batch, was not in his final user group at the end of the reference test, and his user group at round $t$ is different from his user group that was returned by the reference test, i.e., 
\begin{align*}
\ca{T}_u(t_{\ca{B}_{\ell-1}}+1+\s{T_{static}}) \neq \ca{P}_{\s{Test}_0}(u)
\quad \text{and} \quad 
\ca{T}_u(t) \neq \ca{P}_{\s{Test}_0}(u).
\end{align*}
\item A user $u$ who satisfies $\s{L}_u=1,\s{B}_{u,t}=1,\s{X}_{u,2}=1,\s{J}_{u,t}=1$ and $\s{Z}_{t} = 0$, is one who has changed his user group in the first $\s{T_{static}}$ rounds in the batch, and his user group at the end of the reference test is also same as the one provided by the estimate of the reference test, but his user group at round $t$ is different from his user group at the end of the reference test, i.e., 
\begin{align*}
\ca{T}_u(t_{\ca{B}_{\ell-1}}+1+\s{T_{static}}) = \ca{P}_{\s{Test}_0}(u)
\quad \text{and} \quad 
\ca{T}_u(t) \neq \ca{P}_{\s{Test}_0}(u).
\end{align*}
\item A user $u$ who satisfies $\s{L}_u=0,\s{B}_{u,t}=1,\s{J}_{u,t}=1$ and $\s{Z}_{t} = 0$ is one who has not changed his user group in the first $\s{T_{static}}$ rounds in the batch, but his user group at round $t$ is different from his user group at the beginning of the batch, i.e., 
\begin{align*}
&\ca{T}_u(t_{\ca{B}_{\ell-1}}+1+\s{T_{static}}) = \ca{T}_u(t), \quad \text{for}\;t \in\calT_{\s{test},\ell},\\
&\ca{T}_u(t) \neq \ca{P}_{\s{Test}_0}(u).
\end{align*}  
\end{itemize}
Given the above three sub-cases, it is clear that $\s{H}_t$ for $t\in\calB_\ell\setminus\calT_{\s{test},\ell}$, can be written as
\begin{align}
\s{H}_t &= \mathds{1}\left[\sum_{u \in [\sN]}\mathds{1}\left[\s{L}_u=1,\s{B}_{u,t}=1,\s{X}_{u,3}=1,\s{J}_{u,t}=1,\s{Z}_{t} = 0 \right]\right.\nonumber\\
&\quad\quad\quad+ \sum_{u \in [\sN]}\mathds{1}\left[\s{L}_u=0,\s{B}_{u,t}=1,\s{J}_{u,t}=1, \s{Z}_{t} = 0 \right] \nonumber\\
&\left.\quad\quad\quad+\sum_{u \in [\sN]} \mathds{1}\left[\s{L}_u=1,\s{B}_{u,t}=1,\s{X}_{u,2}=1,\s{J}_{u,t}=1, \s{Z}_{t} = 0 \right]>0\right].
\end{align}
Basically, $\s{H}_t$ indicates whether at time $t\in\calB_{\ell}\setminus\calT_{\s{test},\ell}$ a bad user is present or not. Accordingly, we bound the regret in this case as follows
\begin{align}
\s{A}_4&\triangleq\frac{1}{\sN}\bb{E} \left[\sum_{t \in \calB_{\ell}\setminus\calT_{\s{test},\ell}:\s{H}_t=1} \sum_{u \in [\sN]} \mathds{1}\left[\mathbf{R}_{u,\pi_{u,t}} = 0 \mid \ca{A}^c\right] \right] \\
&\le \bb{E}\sum_{t \in \calB_{\ell}\setminus\calT_{\s{test},\ell}}\mathds{1}\left[\s{H}_t=1\right]\\
& = \bE\sum_{t\in\calB_{\ell}:\exists u\in[\sN],\s{J}_{u,t}=1}\s{G}_t,
\label{eqn:A3_1}
\end{align}
where in \eqref{eqn:A3_1} we sum over all those rounds where some user changed its type, and $\s{G}_t$ counts the number of rounds it takes to detect the bad users. This random variable is clearly stochastically dominated by by a Geometric random variable with mean $1/p$. Indeed, a test can start at every round with probability $p$, and a test that starts at a round $\s{Z}_{t}=0$ will certainly reveal that the user is in a different user group than the one returned in the reference test $\ca{P}_{\s{test}_0}$. Accordingly, we will add that user to the set $\ca{V}_{t_{\calB_{\ell-1}}+1+\s{T_{static}}+t}$. Therefore, we obtain that,
\begin{align}
\s{A}_4&\leq \bE\sum_{t\in\calB_{\ell}:\exists u\in[\sN],\s{J}_{u,t}=1}\s{G}_t\leq \bE\sum_{t\in\calB_{\ell}:\exists u\in[\sN],\s{J}_{u,t}=1}\frac{1}{p}\leq \frac{\s{V}_{\calB_{\ell},1}}{p}.
\label{eqn:A3}
\end{align}
\subsection{The ``Static" Regret}
It remains to bound the regret for those round where we do not test and all bad users are detected, i.e.,
\begin{align}
\s{A}_5\triangleq\frac{1}{\sN}\bb{E} \left[\sum_{t \in  \calB_{\ell}\setminus\calT_{\s{test},\ell}:\s{H}_t=0,\s{Y}_t = 0} \sum_{u \in [\sN]\setminus \ca{V}_t} \mathds{1}\left[\mathbf{R}_{u,\pi_{u,t}} = 0 \mid \ca{A}^c\right] \right].
\end{align}
We shall refer to this regret as the \textit{static} regret. This static case was studied in \cite{Bresler:2014}, where algorithm \textsc{Recommend} was analyzed thoroughly. As discussed before, in \cite{Bresler:2014} it was assumed that users of the same user-type have the same exact preference vectors, while in this paper we assume the weaker coherence Assumption~\textbf{A3}. Nonetheless, except for a few technical differences (which we highlight in the proof of the following result), our analysis relies on the proof of Theorem~1 in \cite{Bresler:2014}.
\begin{lemma}[No Variations]\label{thm:2}
Let $\delta\in(0,1)$, and consider the latent source model and assumptions \textbf{A1}--\textbf{A3}. Also, assume that $\sN=\Omega\p{\frac{\sM}{\nu}\log\frac{1}{\delta}+\p{\frac{3}{\delta}}^{1/\alpha}}$. Then, for any $\s{T_{static}}\leq\Delta_{\sT}\leq\mu\cdot\sM$, we have
\begin{align}
\s{A}_5&\leq (\Delta_{\sT}-\s{T_{static}})\cdot\delta.
\end{align}
\end{lemma}

\subsection{Collecting Terms}

We finally collect all the above bounds to obtain the result stated in Theorem~\ref{thm:1}. Specifically, using \eqref{eqn:A0}, \eqref{eqn:A1}, \eqref{eqn:A2}, \eqref{eqn:A3}, and Theorem~\ref{thm:1}, we obtain
\begin{align}
\mathsf{regret}(\calB_\ell)&\leq \s{T_{static}}\cdot(1-\mu)+(\Delta_{\sT}-\s{T_{static}})\cdot\delta+2\cdot\s{V}_{\ca{B}_{\ell},1}\cdot\s{T_{static}}+p\cdot\Delta_{\sT}\cdot\s{T_{static}}\nonumber\\
&\quad+\frac{\s{V}_{\ca{B}_{\ell},1}}{p}+\frac{3\Delta_{\sT}\s{V}_{\ca{B}_{\ell},2}}{2\nu}\label{eqn:ubl0}\\
&\leq\delta\cdot\Delta_{\sT}+\s{T_{static}}\cdot(1-\delta-\mu)+2\cdot\s{V}_{\ca{B}_{\ell},1}\cdot\s{T_{static}}+p\cdot\Delta_{\sT}\cdot\s{T_{static}}\nonumber\\
&\quad+\frac{\s{V}_{\ca{B}_{\ell},1}}{p}+\frac{3\Delta_{\sT}\s{V}_{\ca{B}_{\ell},2}}{2\nu},\label{eqn:ubl}
\end{align}
where the first term at the r.h.s. of \eqref{eqn:ubl0} is the regret due to the first $\s{T_{static}}$ rounds where we recommend random items. Since \eqref{eqn:ubl} is true for every batch $\calB_\ell$, we can sum-up over $\ell$, and obtain that
\begin{align}
\mathsf{regret}(\sT)&\leq\sum_{\ell=1}^{\ceil{\sT/\Delta_{\mathsf{T}}}}\mathsf{regret}(\calB_\ell)\\
&\leq \delta\cdot\sT+\frac{\sT}{\Delta_{\sT}}\s{T_{static}}\cdot(1-\delta-\mu)+2\cdot\s{V_{1}}\cdot\s{T_{static}}+p\cdot\sT\cdot\s{T_{static}}+\frac{\s{V}_1}{p}+\frac{3\Delta_{\sT}\s{V_2}}{2\nu}.
\end{align}
Minimizing the r.h.s. of the above inequality w.r.t. $p$, we obtain that its optimal value is $p^{\star}=\sqrt{\s{V}_1/(\s{T}\cdot\s{T_{static}})}$. Therefore,
\begin{align}
\mathsf{regret}(\sT)&\leq \delta\cdot\sT+\frac{\sT}{\Delta_{\sT}}\s{T_{static}}\cdot(1-\delta-\mu)+2\cdot\s{V_{1}}\cdot\s{T_{static}}+2\sqrt{\s{V}_1\cdot\sT\cdot\s{T_{static}}}+\frac{3\Delta_{\sT}\s{V}_2}{2\nu }.\label{eqn:RegBeforeDelta}
\end{align}
It is left to do is to minimize the r.h.s. of the above inequality over $\Delta_\mathsf{T}$. The optimal value is given in a form of a solution for a cubic equation. Alternatively, it turns out that the following choice which minimizes the first three terms at the r.h.s. of \eqref{eqn:RegBeforeDelta} is
\begin{align}
\Delta_{\mathsf{T}}^* = \min\p{\sT,\sqrt{\frac{2\nu\sT}{3\s{V_{2}}}\kappa}},
\end{align}
where $\kappa\triangleq\s{T_{static}}(1-\delta-\mu)$. Substituting this value back in \eqref{eqn:RegBeforeDelta} gives
\begin{align}
\mathsf{regret}(\sT)&\leq \delta\cdot\sT+\max\p{\kappa,\sqrt{\frac{3\s{V_{2}\sT}\kappa}{2\nu}}}+2\cdot\s{V_{1}}\cdot\s{T_{static}}+2\sqrt{\s{V}_1\cdot\sT\cdot\s{T_{static}}}\nonumber\\
&\quad\quad+\min\p{\frac{3\s{V_{2}}\sT}{2\nu},\sqrt{\frac{3\s{V}_{2}\sT\kappa}{2\nu}}},
\end{align}
and so $\mathsf{reward}(\sT) =\sT-\mathsf{regret}(\sT)$ is lower bounded by the same expression as in Theorem~\ref{thm:1}. Note that the condition $\Delta_{\mathsf{T}}>\s{T_{static}}$ in Lemma~\ref{thm:2} boils down to $\sT>\mathsf{T}_{\mathsf{static}}\cdot \max\ppp{1,\frac{\mathsf{3V_{2}}}{2\nu(1-\delta-\mu)}} = \mathsf{T}_{\mathsf{learn}}$. Finally, for $\sT\leq\mathsf{T}_{\mathsf{learn}}$, we get that $\mathsf{reward}(\sT)\geq \mu\cdot\sT$, as claimed.

\subsection{Proof of Lemma~\ref{thm:2}}

To prove the result in Lemma~\ref{thm:2}, it is suffice to lower bound the probability $\pr\pp{\mathbf{R}_{u\pi_{u,t}}=1,\s{Y}_t=0,\s{H}_t=0}$. To that end, for any $u\in[\sN]$ and $t\in[\sT]$, define
\begin{align}
\calG_{u,t}\triangleq\ppp{|\partial_t(u)|\geq\frac{2\nu\sN}{3}},\label{eqn:GoodSetDef}
\end{align}
where $\partial_t(u)$ is the set of neighbors at time $t$ user $u$ have from the same user-types, respectively. For $t$ large enough the probability of $\calG_{u,t}$ is lower bounded strictly by zero. To show that recall that $|\calT_u(t)|$ is the number of users in user's $u$ type at round $t$. As we argued above at each round we know that $|\calT_u(t)|>\frac{2\nu\sN}{3}$. Also, recall that in the beginning of the batch we devote the first $\s{T_{static}}$ recommendations for creating an initial partition $\calP_0$ of the users into types (see, the the fourth step in Algorithm~\ref{algo:user_user_CF_noise}). We showed in Lemma~\ref{lem:test} that the resulted partition is correct with probability at least $1-\delta/3$, and therefore, $|\partial_t(u)|\frac{2\nu\sN}{3}$ with the same probability, i.e., $\pr[\calG_{u,t}]\geq1-\delta/3$, for $t\in\calB_{\ell}\setminus\calT_{\s{test},\ell}$. 

Next, using the same steps as in the proof of Lemma~2 in \cite{Bresler:2014}, we show that the good neighborhoods have, through random exploration, accurately estimated the probability of liking each item. Thus, we correctly classify each item as likable or not with high probability. In particular, we show Below that
\begin{align}
\pr\pp{\left.\mathbf{R}_{u,\pi_{u,t}}=1,\s{Y}_t=0,\s{H}_t=0\right|\calG_{u,t}}&\geq 1-2\sM\exp\p{-2\frac{\Delta^2\nu t\sN^{1-\alpha}}{3\sM}}-\frac{1}{\sN^\alpha}.\label{eqn:ProbEqual0}
\end{align}
Before proving the above inequality let us first show how we can use it lower bound the regret. Indeed, combining the above inequality with the fact that $\pr[\calG_{u,t}]\geq1-\delta/3$, we get
\begin{align}
\pr\pp{\mathbf{R}_{u,\pi_{u,t}}=1,\s{Y}_t=0,\s{H}_t=0}&\geq 1-2\sM\exp\p{-2\frac{\Delta^2\nu t\sN^{1-\alpha}}{3\sM}}-\frac{1}{\sN^\alpha}-\frac{\delta}{3}.\label{eqn:ProbEqual1}
\end{align}
It can be seen that if the number of users $\sN$ satisfy $\sN=\Omega\p{\frac{\sM}{\nu}\log\frac{1}{\delta}+\p{\frac{3}{\delta}}^{1/\alpha}}$, and of course $t\geq\s{T_{static}}$, then the r.h.s. of \eqref{eqn:ProbEqual1} is at least $1-\delta$, namely, $\pr\pp{\mathbf{R}_{u,\pi_{u,t}}=1,\s{Y}_t=0,\s{H}_t=0}\geq1-\delta$. Therefore, we obtain,
\begin{align}
\s{A}_5&\leq\sum_{t\in\calB_{\ell}\setminus\calT_{\s{test},\ell}}\frac{1}{\sN}\sum_{u=1}^{\sN}\pr\pp{\mathbf{R}_{u,\pi_{u,t}}=0,\s{Y}_t=0,\s{H}_t=0\vert\calA^c}\nonumber\\
&\leq (\Delta_{\sT}-\s{T_{static}})\cdot\delta,
\end{align}
where in the second inequality we have used Assumption~\textbf{A2}. Next, we prove \eqref{eqn:ProbEqual0}. First, we lower bound the number of times an arbitrary item has been rated by the good neighbors of some user $u$, conditioned on the event $\calG_{u,t}$. To that end, note that the number of good neighbors user $u$ has and who have rated item $i$ is stochastically dominated by $\mathsf{Binomial}\p{\frac{2\nu\sN}{3},\frac{t}{\sM\sN^\alpha}}$. Let $\calD$ be the event ``item $i$ has less than $\frac{\nu t\sN^{1-\alpha}}{3\sM}$ ratings from good neighbors of $u$". Then, Chernoff's bound then gives
\begin{align}
\pr\p{\calD}&\leq\pr\p{\mathsf{Binomial}\p{\frac{2\nu\sN}{3},\frac{t}{\sM\sN^\alpha}}\leq\frac{\nu t\sN^{1-\alpha}}{3\sM}}\\
&\leq\exp\p{-\frac{\nu t\sN^{1-\alpha}}{3\sM}}.\label{eqn:calD}
\end{align}
Next, conditioned on $\calG_{u,t}$ and $\calD$ we prove that with high probability when exploiting the algorithm predicts correctly every item as likable or unlikable for user $u$. Recall our definition for the posterior $\hat{p}_{u\ell}$ in \eqref{eqn:score}. Suppose item $i$ is likeable by user $u$, and let $\s{G}\triangleq\frac{\nu t\sN^{1-\alpha}}{3\sM}$. Then, conditioned on $\s{G}$, $\hat{p}_{u\ell}$ stochastically dominates $\tilde{p}_{ui}\triangleq\mathsf{Binomial}(\s{G},p_{ui})/\s{G}$. Then,
\begin{align}
\pr\p{\left.\tilde{p}_{ui}\leq\frac{1}{2}\right|\s{G}} &= \pr\p{\left.\mathsf{Binomial}(\s{G},p_{ui})\leq\frac{\s{G}}{2}\right|\s{G}}\\
& \leq\exp\p{-2\s{G}\Delta^2}\\
&\leq \exp\p{-2\frac{\Delta^2\nu t\sN^{1-\alpha}}{3\sM}},
\end{align}
where the first inequality follows from Hoeffding's inequality, and the second inequality is because $p_{ui}\geq1/2+\Delta$. Using monotonicity, we also have
\begin{align}
\pr\p{\left.\tilde{p}_{ui}\leq\frac{1}{2}\right|\s{G}\geq\frac{\nu t\sN^{1-\alpha}}{3\sM}} &\leq \exp\p{-2\frac{\Delta^2\nu t\sN^{1-\alpha}}{3\sM}}.
\end{align}
Using the same steps we can show that if item $i$ is unlikeable by user $u$ then with the same probability $\tilde{p}_{ui}\geq\frac{1}{2}$. Taking a union bound over all items we get that with probability at least $1-\sM\exp\p{-2\frac{\Delta^2\nu t\sN^{1-\alpha}}{3\sM}}$ our algorithm correctly classifies every item as likable or unlikable for user $u$. We are now in a position to prove \eqref{eqn:ProbEqual0}. Specifically, for user $u$ at time $t$, conditioned on $\calG_{u,t}$ we have shown in \eqref{eqn:calD} that with probability at least $1-\sM\exp\p{-\frac{\nu t\sN^{1-\alpha}}{3\sM}}$ \emph{every} item has more than $\frac{\nu t\sN^{1-\alpha}}{3\sM}$ ratings from good neighbors of $u$. Now, using the fact that with probability at least $1-\sM\exp\p{-2\frac{\Delta^2\nu t\sN^{1-\alpha}}{3\sM}}$ we classify correctly all items, coupled with the fact that we exploit with probability $1-\sN^{-\alpha}$, we get
\begin{align}
\pr\pp{\left.\mathbf{R}_{u,\pi_{u,t}}=1,\s{Y}_t=0,\s{H}_t=0\right|\calG_{u,t}}&\geq 1-\sM\exp\p{-\frac{\nu t\sN^{1-\alpha}}{3\sM}}-\sM\exp\p{-2\frac{\Delta^2\nu t\sN^{1-\alpha}}{3\sM}}\nonumber\\
&\quad\quad-\frac{1}{\sN^\alpha}\\
&\geq 1-2\sM\exp\p{-2\frac{\Delta^2\nu t\sN^{1-\alpha}}{3\sM}}-\frac{1}{\sN^\alpha},
\end{align}
as claimed.

\end{appendices}

\end{document}